%% file: main.tex
\documentclass[10pt, conference, letterpaper]{IEEEtran}

\newcommand{\ar}[1]{\textcolor{black}{#1}}
\newcommand{\ff}[1]{\textcolor{black}{#1}}
\newcommand{\om}[1]{\textcolor{black}{#1}}
\newcommand{\gn}[1]{\textcolor{black}{#1}}
\newcommand{\el}[1]{\textcolor{black}{#1}}

\newcommand{\arx}[1]{\textcolor{black}{#1}}

\usepackage{parskip}

\usepackage{enumitem}
\setlist[itemize]{topsep=0pt,itemsep=-1ex,partopsep=1ex,parsep=1ex}
\setlist[enumerate]{topsep=0ex,itemsep=0pt,partopsep=0ex,parsep=0ex}

\usepackage{amsmath, amssymb, amsfonts, amsthm, empheq, mathtools, bm, dsfont, cases, textcomp}

\usepackage{amsthm}
    {
        \theoremstyle{plain}
        \newtheorem{assumption}{Assumption}
        \newtheorem{theorem}{Theorem}
        \newtheorem{lemma}{Lemma}

        \theoremstyle{definition}
        
        \theoremstyle{plain}
    }
    
\newcommand{\norm}[1]{\left\lVert#1\right\rVert}

\newcommand\scalar[2]{\langle #1, #2 \rangle}
\newcommand{\proj}[2]{\mathbf{Proj}_{#1}{(#2)}}

\DeclareMathOperator*{\E}{\mathbb{E}}

\DeclarePairedDelimiter\abs{\lvert}{\rvert}

\DeclareMathOperator{\diag}{diag}
\DeclareMathSymbol{\shortminus}{\mathbin}{AMSa}{"39}

\newcommand{\fedavg}{\texttt{FedAvg}}
\newcommand{\cafed}{\texttt{CA-Fed}}
\newcommand{\fast}{\texttt{F3AST}}
\newcommand{\adafed}{\texttt{AdaFed}}
\newcommand{\unbiased}{\texttt{Unbiased}}

\usepackage{graphicx}
\usepackage{caption}
\usepackage{subcaption}

\usepackage{tikz}
\usetikzlibrary{automata,arrows,positioning,calc}

\usepackage{booktabs}
\usepackage{multirow}

\usepackage[ruled, vlined, linesnumbered, noend]{algorithm2e}

\usepackage{optidef}

\usepackage{cite}
\usepackage[hidelinks]{hyperref}

\usepackage{comment}
\usepackage{xcolor}
\usepackage{soul}

\def\BibTeX{{\rm B\kern-.05em{\sc i\kern-.025em b}\kern-.08em
    T\kern-.1667em\lower.7ex\hbox{E}\kern-.125emX}}

\AtBeginDocument{
 \addtolength{\belowcaptionskip}{-2.5pt}
}

\IEEEoverridecommandlockouts

\begin{document}

\title{Federated Learning under Heterogeneous and Correlated Client Availability\vspace{-0.05in}}

\author{
\IEEEauthorblockN{Angelo Rodio\IEEEauthorrefmark{1}, Francescomaria Faticanti\IEEEauthorrefmark{1}, Othmane Marfoq\IEEEauthorrefmark{1}\IEEEauthorrefmark{2}, Giovanni Neglia\IEEEauthorrefmark{1}, Emilio Leonardi\IEEEauthorrefmark{3}}
\IEEEauthorblockA{
\IEEEauthorrefmark{1}Inria, Université Côte d’Azur, France. Email: \{firstname.lastname\}@inria.fr, \\
\IEEEauthorrefmark{2}Accenture Labs, Sophia-Antipolis, France. Email: \{firstname.lastname\}@accenture.com, \\
\IEEEauthorrefmark{3}Politecnico di Torino, Turin, Italy. Email: \{firstname.lastname\}@polito.it}
\thanks{This research was supported by the French government through the 3IA Côte d’Azur Investments in the Future project by the National Research Agency (ANR) with reference ANR-19-P3IA-0002, and \gn{by Groupe La Poste, sponsor of Inria Foundation, in the framework of FedMalin Inria Challenge}.}
\thanks{\arx{\underline{A first version of this work has been accepted at IEEE INFOCOM 2023.}}}
\vspace{-0.34in}
}

\maketitle

\thispagestyle{plain}
\pagestyle{plain}

\begin{abstract}
\input{abstract.tex}
\end{abstract}


\section{Introduction}
\label{sec:intro}
\input{intro}

\section{Background and Related Works}
\label{sec:problem}
\input{problem}

\section{Analysis}
\label{sec:analysis}
\input{analysis}

\section{Proposed Algorithm}
\label{sec:algorithm}
\input{algorithm}

\section{Experimental Evaluation}
\label{sec:experiments}
\input{experiments}

\section{Discussion}
\label{sec:discussion}
\input{discussion.tex}

\section{Conclusion}
\label{sec:conclusion}
\input{conclusion.tex}

\appendix
\input{appendix}

\clearpage  
\bibliographystyle{IEEEtran}
\bibliography{IEEEabrv,references.bib}

\end{document}

%% file: abstract.tex
The enormous amount of data produced by mobile and IoT devices has motivated the development of federated learning (FL), a framework allowing such devices (or clients) to collaboratively train machine learning models without sharing their local data. 
FL algorithms (like \fedavg{}) iteratively aggregate model updates computed by clients on their own datasets. 
Clients may exhibit different levels of participation, often correlated over time and with other clients.
This paper presents the first convergence analysis for a \fedavg{}-like FL algorithm under heterogeneous and correlated client availability. 
Our analysis highlights how correlation adversely affects the algorithm's convergence rate and how the aggregation strategy can alleviate this effect at the cost of steering training toward a biased model.
Guided by the theoretical analysis, we propose \cafed{}, a new FL algorithm that tries to balance the conflicting goals of maximizing convergence speed and minimizing model bias.
To this purpose, \cafed{} dynamically adapts the weight given to each client and may ignore clients with low availability and large correlation.
Our experimental results show that \cafed{} achieves higher time-average accuracy and a lower standard deviation than state-of-the-art \adafed{} and \fast{}, \ar{both on synthetic and real datasets}.


%% file: intro.tex
The enormous amount of data generated by mobile and IoT devices motivated the emergence of distributed machine learning training paradigms~\cite{verbraeken2020survey,wang2018edge}. Federated Learning (FL)~\cite{konevcny2016federated, mcmahan2017communication, kairouz2021advances, li2020challenges} is an emerging framework where geographically distributed devices (or clients) participate in the training of a shared Machine Learning (ML) model without sharing their local data. FL was proposed to reduce the overall cost of collecting a large amount of data as well as to protect potentially sensitive users' private information. In the original Federated Averaging algorithm (\fedavg{})~\cite{mcmahan2017communication}, a central server selects a random subset of clients from the set of available clients and broadcasts them the shared model. The sampled clients perform a number of independent Stochastic Gradient Descent (SGD) steps over their local datasets and send their local model updates back to the server. Then, the server aggregates the received client updates to produce a new global model, and a new training round begins. 
\ar{\el{At} each iteration of \fedavg{}, the server typically samples  randomly a few hundred devices to participate~\cite{eichner2019semi,wang2021field}}.

In real-world scenarios, the availability/activity of clients is dictated by exogenous factors that are beyond the control of the orchestrating server and hard to predict. 
For instance, only smartphones that are idle, under charge, and connected to broadband networks are commonly allowed to participate in the training process~\cite{mcmahan2017communication,bonawitz2019towards}. These eligibility requirements can make the availability of devices correlated over time and space \cite{eichner2019semi, ding2020distributed, zhu2021diurnal, doan2020local}. 
For example, \emph{temporal correlation} may origin from a smartphone being under charge for a few consecutive hours and then ineligible for the rest of the day. Similarly, the activity of a sensor powered by renewable energy may depend on natural phenomena intrinsically correlated over time (e.g., solar light).
\emph{Spatial correlation} refers instead to correlation across different clients, which  often emerges as consequence of users' \ar{different} geographical distribution.
For instance, clients in the same time zone often exhibit similar availability patterns, e.g., due to time-of-day effects.

Temporal correlation in the data sampling procedure is known to negatively affect the performance of ML training even in the centralized setting~\cite{doan2020convergence, sun2018markov} and can potentially lead to \emph{catastrophic forgetting}: the data used during the final training phases can have a disproportionate effect on the final model, ``erasing'' the memory of previously learned information~\cite{MCCLOSKEY1989catastrophic, kirkpatrick2017overcoming}. Catastrophic forgetting has also been observed in FL, where clients in the same geographical area have more similar local data distributions and clients' participation follows a cyclic daily pattern (leading to spatial correlation)~\cite{eichner2019semi, ding2020distributed, zhu2021diurnal,tang2022fedcor}. Despite this evidence, a theoretical study of the convergence of FL algorithms under \ar{both} temporally and spatially correlated client participation is still missing. 

This paper provides the first convergence analysis of \fedavg{}~\cite{mcmahan2017communication} under heterogeneous and correlated client availability. We assume that clients' temporal and spatial availability follows an arbitrary finite-state Markov \el{chain}: this assumption models a realistic scenario in which the activity of clients is correlated and, at the same time, still allows the analytical tractability of the system.
Our theoretical analysis (i)~quantifies the negative effect of correlation on the algorithm's convergence rate through an additional term, \el{which depends}
on the spectral properties
of the Markov chain; 
(ii)~points out a trade-off between two conflicting objectives: slow convergence to the optimal model, or fast convergence to a biased model, i.e., a model that minimizes an objective function different from the initial target. 
Guided by insights from the theoretical analysis, we propose \cafed{}, an algorithm which dynamically assigns weights to clients and \el{achieves a good} trade-off between maximizing convergence speed and minimizing model bias. 
\gn{Interestingly, \cafed{}} can decide to ignore clients with low availability and \gn{high temporal correlation}. 
Our experimental results demonstrate that \gn{excluding \el{such} clients}
is a \gn{simple, but} effective approach to handle the heterogeneous and correlated client availability in FL. Indeed, while \cafed{} achieves a comparable maximum accuracy as the state-of-the-art methods \fast{}~\cite{ribero2022federated} and \adafed{}~\cite{tan2022adafed}, \ar{its test accuracy exhibits higher time-average and smaller variability over time}. 

The remainder of this paper is organized as follows. Section~\ref{sec:problem} describes the problem of correlated client availability in FL and discusses the main related works. Section~\ref{sec:analysis} provides a convergence analysis of \fedavg{} under heterogeneous and correlated client participation. \cafed{}, our correlation-aware FL algorithm, is presented in Section~\ref{sec:algorithm}. We evaluate \cafed{} in Section~\ref{sec:experiments}, comparing it with state-of-the-art methods on synthetic and real-world data. Section~\ref{sec:conclusion} concludes the paper.

%% file: problem.tex
We consider a finite set $\mathcal{K}$ of $N$~clients. Each client $k \in \mathcal{K}$ holds a local dataset $D_k$. Clients aim to jointly learn the parameters \gn{$\bm{w} \in W \subseteq \mathbb{R}^{d}$} of a global ML model (e.g., the weights of a neural network architecture). During training, the quality of the model with parameters $\bm{w}$ on a \ar{data sample $\xi \in D_k$ is measured} by a loss function $f(\bm{w}; \xi)$.
The  clients solve, under the orchestration of a central server, the following optimization problem:
\begin{gather}
    \label{opt:target_objective}
    \min_{\bm{w} \in W\subseteq \mathbb{R}^{d}}\left[F(\bm{w}) \coloneqq \sum_{k \in \mathcal{K}} \alpha_k F_k(\bm{w}) \right],
\end{gather}
where $F_k(\bm{w})\coloneqq \frac{1}{|D_k|} \sum_{\xi \in D_k} f(\bm{w}; \xi)$ is the average loss computed on client $k$'s local dataset, and \ar{$\bm{\alpha} = (\alpha_k)_{k \in \mathcal{K}}$} are positive coefficients such that $\sum_{k}\alpha_{k}=1$. \ar{They represent the} \gn{ \emph{target importance}} \ar{assigned by the central server to each client~$k$}. Typically $(\alpha_k)_{k \in \mathcal{K}}$ are set proportional to the clients' dataset size $|D_{k}|$, such that the objective function $F$ in~\eqref{opt:target_objective} coincides with the average loss computed on the union of the clients' local datasets \ar{$D = \cup_{k \in \mathcal{K}} D_k$}. 

Under proper assumptions, precised in Section~\ref{sec:analysis}, \ff{Problem~\eqref{opt:target_objective}} admits a \el{unique} solution. We use $\bm{w}^{*}$ (resp. $F^{*}$) to denote \el{the} minimizer (resp. the minimum value) of $F$. Moreover, for $\ar{k{\in}\mathcal{K}}$, $F_{k}$ admits a \el{unique} minimizer on $W$. We use $\bm{w}_{k}^{*}$ (resp. $F_{k}^{*}$) to denote \ar{the} minimizer (resp. the minimum value) of $F_{k}$.

\ff{Problem~\eqref{opt:target_objective}} is commonly solved through iterative algorithms~\cite{mcmahan2017communication, wang2021field} requiring multiple communication rounds between the server and the clients. At round $t>0$, the server broadcasts the latest estimate of the global model $\bm{w}_{t, 0}$ \ar{to the set of available clients} \gn{($A_{t}$)}.
Client $k \in A_t$ updates the global model with its local data through $E\geq1$ steps of local Stochastic Gradient Descent (SGD):
\begin{align}
    \textstyle
    \bm{w}_{t,j+1}^k = \bm{w}_{t,j}^k - \eta_t \nabla F_k (\bm{w}_{t,j}^k, \mathcal{B}_{t,j}^k) \quad j=0,\dots,E-1,
    \label{eq:localSGD}
\end{align} 
where $\eta_t>0$ is an appropriately chosen learning rate, referred to as \emph{local learning rate}; $\mathcal{B}_{t,j}^k$ is a random batch sampled from client $k$' local dataset \ar{at round $t$ and step $j$}; $\nabla F_k (\cdot, \mathcal{B})\coloneqq\frac{1}{|\mathcal{B}|}\sum_{\xi \in\mathcal{B}} \nabla f(\cdot, \xi)$ is an unbiased estimator of the local gradient $\nabla F_{k}$. Then, each client sends its local model update $\Delta^{k}_{t} \coloneqq \bm{w}_{t,E}^k - \bm{w}_{t,0}^k$ to the server. The server computes $\Delta_{t} \coloneqq \sum_{k \in A_t} q_k \cdot \Delta_{t}^{k}$, a weighted average of the clients' local updates with \ar{non-negative \emph{aggregation weights} $\bm{q} = (q_{k})_{k\in \mathcal{K}}$. The choice of the aggregation weights defines an aggregation strategy (we will discuss different aggregation strategies later)}.
The aggregated update $\Delta_{t}$ can be interpreted as a proxy for $-\nabla F(\bm{w}_{t, 0})$; the server applies it to the global model:
\begin{align}
    \textstyle
    \bm{w}_{t+1,0} = \proj{W}{\bm{w}_{t,0} + \eta_{s} \cdot \Delta_{t}}
    \label{eq:fedavg_global_aggregation}
\end{align}
where \ar{$\proj{W}{\cdot}$ denotes the projection over the set $W$}, and $\eta_{s}>0$ is an appropriately chosen learning rate, referred to as the \emph{server learning rate}.\footnote{
\ar{The aggregation rule~\eqref{eq:fedavg_global_aggregation} has been considered also in other works, e.g.,~\cite{nichol2018first, reddi2021adaptive, wang2021field}. In other FL algorithms, the server computes an average of clients’ local models. This aggregation rule can be obtained with minor changes to~\eqref{eq:fedavg_global_aggregation}.}}

The aggregate update $\Delta_{t}$ is, in general, a biased estimator of $\ar{-}\nabla F(\bm{w}_{t, 0})$, where each client $k$ is taken into account proportionally to its frequency of appearance in the set~$A_t$ and to its aggregation weight~$q_k$. 
Indeed, under proper assumptions specified in Section~\ref{sec:analysis}, one can show~(see Theorem~\ref{theorem:bias}) that the update rule described by \eqref{eq:localSGD} and \eqref{eq:fedavg_global_aggregation} converges to \el{the unique} minimizer of a biased global objective $F_B$, \el{which} depends both \ar{on the clients' availability} \gn{(i.e., on the sequence $\left(A_{t}\right)_{t>0}$)} \ar{and on the aggregation strategy} \gn{(i.e., on  $\bm{q}=\left(q_{k}\right)_{k\in\mathcal{K}}$)}: 
\begin{align}
    \textstyle
    F_B(\bm{w}) \coloneqq \sum_{k=1}^N p_k F_k(\bm{w}),
    ~\text{with} ~p_k \coloneqq \frac{\pi_{k} q_{k}}{\sum_{h=1}^N \pi_{h} q_{h}},
    \label{def:biased_objective}
\end{align}
where $\pi_k \coloneqq \lim_{t \rightarrow \infty}  \mathbb{P}(k \in A_t)$ is the asymptotic availability of client $k$. \ar{The coefficients $\bm{p} = (p_{k})_{k \in \mathcal{K}}$ can be interpreted as the}
\gn{\emph{biased importance} the server is giving}
\ar{to each client~$k$ during training, in general different from the \emph{\ar{target} importance} $\bm{\alpha}$.} In what follows,  $\bm{w}_{B}^{*}$ (resp. $F_{B}^{*}$) denotes \el{the}   minimizer (resp. the minimum value) of $F_B$.

In some large-scale FL applications, like training Google keyboard next-word prediction models, each client participates in training at most for one round. The orchestrator usually selects a few hundred clients at each round for a few thousand rounds (e.g., see~\cite[Table~2]{kairouz2021advances}), but the available set of clients may include hundreds of millions of Android devices. In this scenario, it is difficult to address the potential bias unless there is some a-priori information about each client's availability. 
Anyway, FL can be used by service providers with access to a much smaller set of clients (e.g., smartphone users that have installed a specific app). In this case, a client participates multiple times in training: the orchestrating server may keep track of each client's availability and try to compensate for the potentially dangerous heterogeneity in their participation. 

Much previous effort on federated learning~\cite{mcmahan2017communication, li2019convergence, li2020federated, chen2020optimal, fraboni2021clustered, tang2022fedcor, tan2022adafed, ribero2022federated} considered this problem and, under different assumptions on the clients' \gn{availability (i.e., on $\left(A_{t}\right)_{t>0}$)}, 
designed aggregation strategies that unbias $\Delta_{t}$ through an appropriate choice of $\bm{q}$.
Reference~\cite{li2019convergence} provides the first analysis of \fedavg{} on non-iid data under clients' partial participation. 
\gn{Their analysis covers both the case when active clients are sampled uniformly at random without replacement from $\mathcal{K}$ and assigned aggregation weights equal to their \ar{target} importance (as assumed in~\cite{mcmahan2017communication}), and the case when active clients are sampled iid with replacement from $\mathcal{K}$ with probabilities $\bm{\alpha}$ and assigned equal weights (as assumed in~\cite{li2020federated}).
However, references~\cite{mcmahan2017communication, li2019convergence, li2020federated} ignore the variance induced by the clients stochastic availability.} The authors of~\cite{chen2020optimal} reduce such variance by considering only the clients with important updates, as measured by the value of their norm. References~\cite{tang2022fedcor} and~\cite{fraboni2021clustered} reduce the aggregation variance through clustered and soft-clustered sampling, respectively. 

Some recent works~\cite{tan2022adafed,ribero2022federated,cho2022towards} do not actively pursue the optimization of the unbiased objective. Instead, they derive bounds for the convergence error and propose heuristics to minimize those bounds, potentially introducing some bias. Our work follows a similar development: we compare our algorithm with~\fast{} \ar{from~\cite{ribero2022federated}} and~\adafed{} from~\cite{tan2022adafed}.

The novelty of our study is \el{in} considering the spatial and temporal correlation in clients' availability dynamics. As discussed in the introduction, such correlations are also introduced by clients' eligibility criteria, e.g., smartphones being under charge and connected to broadband networks. The effect of correlation has been ignored until now, probably due to the additional complexity in studying FL algorithms' convergence. To the best of our knowledge, the only exception \ar{is~\cite{ribero2022federated}}, which scratches the issue of spatial correlation by proposing two different algorithms for the case when clients' \ff{availabilities} are uncorrelated and for the case when they are positively correlated (there is no smooth transition from one algorithm to the other as a function of the degree of correlation).

The effect of temporal correlation on \emph{centralized} stochastic gradient methods has been addressed in~\cite{sun2018markov, doan2020convergence,doan2020finite, doan2020local}: these works study a variant of stochastic gradient descent where samples are drawn according to a Markov chain. Reference~\cite{doan2020local} extends its analysis to a FL setting where each client draws samples according to a Markov chain.  
In contrast, our work does not assume a correlation in the data sampling but rather in the client's availability. Nevertheless, some of our proof techniques are similar to those used in this line of work and, in particular, we rely on some results in~\cite{sun2018markov}.

%% file: analysis.tex
\subsection{Main assumptions}

We consider a time-slotted system where a slot corresponds to one FL communication round.  
We assume that clients' availability over the timeslots $t \in \mathbb N$ follows a discrete-time Markov chain $(A_t)_{t\geq0}$.\footnote{
\ar{In Section~\ref{subsection:opt}} we will focus on the case where this chain is the superposition of $N$ independent Markov chains, one for each client.} 
\begin{assumption}
The Markov chain $(A_t)_{t\geq0}$ on the finite state space $[M]$ is time-homogeneous, irreducible, and aperiodic. It has transition matrix \ar{$\bm{P}$} and stationary distribution \ar{$\bm{\pi}$}.
\label{assumption:markov_chain}
\end{assumption}

Markov chains have already been used in the literature to model the dynamics of stochastic networks where some nodes or edges in the graph can switch between active and inactive states~\cite{meyers2021markov, olle1997dynamical}.
The previous Markovian assumption, while allowing a great degree of flexibility, still guarantees the analytical tractability of the system.
The \el{ distance dynamics between  current 
 and stationary distribution of the Markov process} can be characterized by the spectral properties of its transition matrix $\bm{P}$~\cite{levin2017markov}. Let $\lambda_2(\bm{P})$ denote the the second largest eigenvalue of $\bm{P}$ in absolute value. Previous works~\cite{sun2018markov} have shown that: 
\begin{align}
    \textstyle
    \max\limits_{i,j \in [M]}{\abs{[\bm{P}^t]_{i,j} - \pi_j}}\leq C_P \cdot \lambda(\bm{P})^t,
    &&\text{for $t \geq T_P$},
    \label{eq:mc_convergence}
\end{align}
where the parameter $\lambda(\bm{P}) \coloneqq {(\lambda_2(\bm{P})+1)}/{2}$, and $C_P, T_P$ are positive constants whose values are reported in~\cite[Lemma~1]{sun2018markov}.\footnote{Note that \eqref{eq:mc_convergence} holds for different definitions of $\lambda(\bm{P})$ as far as $\lambda(\bm{P}) \in (\lambda_2(\bm{P}), 1)$. The specific choice for $\lambda(\bm{P})$ changes the constants $C_P$ and $T_P$.} 
Note that $\lambda(\bm{P})$~quantifies the correlation of the Markov process $(A_t)_{t\geq0}$: the closer $\lambda(\bm{P})$ is to one, the slower the Markov chain converges to its stationary distribution. 

In our analysis, we make the following additional assumptions. Let $\bm{w}^*, \bm{w}_B^*$ denote the \el{minimizers} of $F$ and  $F_B$ \om{on $W$}, respectively.

\begin{assumption}
The hypothesis class $W$ is convex, compact, and contains \el{in its interior  \ar{the} minimizers} $\bm{w}^*, \bm{w}_B^*, \bm{w}_k^*$.
\label{assumption:space}
\end{assumption}

The following assumptions concern clients' local objective functions $\{F_k\}_{k \in \mathcal K}$. Assumptions \ref{assumption:smoothness} and \ref{assumption:convexity} are standard in the literature on convex optimization \cite[Sections 4.1, 4.2]{bottou2018optimization}. Assumption \ref{assumption:variance} is a standard hypothesis in the analysis of federated optimization algorithms \cite[Section~6.1]{wang2021field}.

\begin{assumption}[L-smoothness]
The local functions $\{F_k\}_{k=1}^N$ have L-Lipschitz continuous gradients: 
$F_k(\bm{v}) \leq F_k(\bm{w}) + \scalar{\nabla F_k(\bm{w})}{\bm{v} - \bm{w}} + \frac{L}{2} \norm{\bm{v} - \bm{w}}_2^2$, $\forall \bm{v}, \bm{w} \in W$.
\label{assumption:smoothness}
\end{assumption}
\begin{assumption}[Strong convexity]
The local functions $\{F_k\}_{k=1}^N$ are $\mu$-strongly convex: 
$F_k(\bm{v}) \geq F_k(\bm{w}) + \scalar{\nabla F_k(\bm{w})}{\bm{v} - \bm{w}} + \frac{\mu}{2} \norm{\bm{v} - \bm{w}}_2^2, ~\forall \bm{v}, \bm{w} \in W$.
\label{assumption:convexity}
\end{assumption}

\begin{assumption}[Bounded variance]
The variance of stochastic gradients in each device is bounded: 
$\E \Vert \nabla F_k (\bm{w}_{t,j}^k, \xi_{t,j}^k) - \nabla F_k(\bm{w}_{t,j}^k) \Vert^2 \leq \sigma_k^2, ~k=1,\dots,N$.
\label{assumption:variance}
\end{assumption}

Assumptions~\ref{assumption:space}--\ref{assumption:variance} imply the following properties for the local functions, described by Lemma~\ref{lem:prop} \ar{(proof in Appendix~\ref{appendix:bias})}.
\begin{lemma} 
    \label{lem:prop}
    Under Assumptions~\ref{assumption:space}--\ref{assumption:variance}, there exist constants $D$, $G$, and $H>0$, such that, for $\bm{w} \in W$ and $k\in\mathcal{K}$, we have:
    \begin{align}
        \Vert \nabla F_k (\bm{w}) \Vert & \leq D, \label{eq:D}
        \\
        \E \Vert \nabla F_k (\bm{w}, \xi) \Vert^2 &\leq   G^2,  \label{eq:G}
        \\
        \abs{F_k(\bm{w}) - F_k(\bm{w}_B^*)} &\leq H.   \label{eq:H}
    \end{align}
\end{lemma}

Similarly to other works \cite{li2019convergence, li2020federated, wang2020tackling, wang2021field}, we introduce a metric to quantify the heterogeneity of clients' local datasets:
\begin{align}
    \Gamma &\coloneqq \max_{k \in \mathcal{K}} \{ F_k(\bm{w}^*) - F_k^* \}.
    \label{eq:gamma}
\end{align}
If the local datasets are identical, the local functions $\{F_k\}_{k \in \mathcal K}$ coincide among them and with $F$, $\bm{w}^*$ is a \el{minimizer} of each local function, and $\Gamma=0$. In general, $\Gamma$ is smaller the closer the distributions the local datasets are drawn from.

\subsection{Main theorems}

Theorem~\ref{theorem:error} (proof in Appendix \ref{appendix:error}) decomposes the error of the \ar{target} global objective as the sum of an optimization error for the biased global objective and a bias error. 

\begin{theorem}[{Decomposing the total error}]
\label{theorem:error}
    Under Assumptions~\ref{assumption:space}--\ref{assumption:convexity}, the optimization error of the \ar{target} global objective $\epsilon = F(\bm{w}) - F^*$ can be bounded as follows:
    \begin{align}
        \epsilon \leq \underbrace{2 \kappa^2 (F_B(\bm{w}) - F_B^*) \vphantom{\chi^2_{\bm{\alpha} \parallel \bm{p}}}}_{\coloneqq \epsilon_{\text{opt}}} 
        + \underbrace{2 \kappa^4 \chi^2_{\bm{\alpha} \parallel \bm{p}} \Gamma}_{\coloneqq \epsilon_{\text{bias}}},
        \label{eq:total_error}
    \end{align}
where $\kappa \coloneqq {L}/{\mu}$, and $\chi^2_{\bm{\alpha} \parallel \bm{p}} \coloneqq \sum_{k=1}^{N} {(\alpha_k - p_k)^2}/{p_k}$.
\end{theorem}
Theorem~\ref{theorem:bias} below proves that the optimization error $\epsilon_{\text{opt}}$ associated to the biased objective $F_B$, evaluated on the trajectory determined by scheme~\eqref{eq:fedavg_global_aggregation}, asymptotically vanishes. The non-vanishing bias error $\epsilon_{\text{bias}}$ captures the discrepancy between $F(\bm{w})$ and $F_B(\bm{w})$. This latter term depends on the chi-square divergence $\chi^2_{\bm{\alpha} \parallel \bm{p}}$ between the \ar{target and biased probability distributions $\bm{\alpha} = (\alpha_k)_{k \in \mathcal{K}}$ and $\bm{p} = (p_k)_{k \in \mathcal{K}}$}, and on $\Gamma$, that quantifies the degree of heterogeneity of the local functions. When all local functions are identical ($\Gamma = 0$), the bias term $\epsilon_{\text{bias}}$ also vanishes.
For $\Gamma >0$, the bias error can still be controlled by the aggregation weights assigned to the devices. In particular, the bias term vanishes when \ar{$q_k \propto \alpha_k / \pi_k, \forall k \in \mathcal{K}$}. \ar{Since it asymptotically cancels the bias error, we refer to this choice as \textit{unbiased aggregation strategy}}. 

However, in practice, FL training is limited to a finite number of iterations $T$ (typically a few hundreds~\cite{eichner2019semi, kairouz2021advances}), and the previous asymptotic considerations may not apply. \ar{In this regime, the unbiased aggregation strategy can be suboptimal, since} \el{the  minimization of $\epsilon_{\text{bias}}$ not necessarily leads
to the minimization of the total error $\epsilon \leq \epsilon_{\text{opt}} + \epsilon_{\text{bias}}$}. This motivates the analysis of the optimization error $\epsilon_{\text{opt}}$.

\begin{theorem}[{Convergence of the optimization error} ${\epsilon}_{\text{{opt}}}$]
Let Assumptions \ref{assumption:markov_chain}--\ref{assumption:variance} hold and the constants $M, L, D, G, H, \Gamma$, $\sigma_k, C_P, T_P, \lambda(\bm{P})$ be defined as above.
Let $Q = \sum_{k \in \mathcal{K}} q_k$. Let the stepsizes satisfy:
\begin{align}
    \textstyle
    \sum_t \eta_t = + \infty, 
    \quad
    \sum_t \ln(t) \cdot \eta_t^2 < + \infty.
    \label{eq:stepsize}
\end{align}
Let $T$ denote the total communication rounds. For $T \geq T_P$, the expected optimization error can be bounded as follows:
\begin{equation}
    \E [F_B(\bar{\bm{w}}_{T,0}) - F_B^*] \leq
    \frac
    {
      \frac
        {\frac{1}{2} \bm{q}^\intercal \bm{\Sigma} \bm{q} + \upsilon}
        {\boldsymbol{\pi}^\intercal \bm{q}} +
      \psi +
      \frac
        {\phi}
        {\ln(1/\lambda(\bm{P}))}
    }
    {(\sum_{t=1}^T \eta_t)},
    \label{eq:convergence_bound}
\end{equation}
where ${\bar{\bm{w}}_{T,0}} \coloneqq \frac{\sum_{t=1}^T \eta_t \bm{w}_{t,0}}{\sum_{t=1}^T \eta_t}$, and
\begin{flalign*}
    \bm{\Sigma} = 
    &\textstyle 
    \diag(\sigma_k^2 \pi_k \sum_t \eta_t^2), &&\\
    \upsilon = 
    &\textstyle 
    \frac{2}{E} \norm{\bm{w}_{0,0} - \bm{w}^*}^2 + \frac{1}{4} {M Q} \sum_t (\eta_t^2 + \frac{1}{t^2}), &&\\
    \psi = 
    &\textstyle 
    ~4L(EQ+2) \Gamma \sum_t \eta_t^2 
    \textstyle 
    + \frac{2}{3} (E-1) (2E-1) G^2 \sum_t \eta_t^2, &&\\
    \ar{\mathcal{J}_t =}
    &\textstyle 
    \ar{\min \left\{ \max \left\{ \left \lceil {\ln \left( 2C_PHt \right)}/{\ln \left( 1/\lambda (\bm{P}) \right)} \right \rceil, T_P \right\}, t \right\}}, &&\\
    \phi =
    &\textstyle 
    ~2 E D G Q \sum_t \ln(2C_PHt) \eta_{t - \mathcal{J}_t}^2.
\end{flalign*}
\label{theorem:bias}
\end{theorem}
\vspace{-2.5ex}
Theorem~\ref{theorem:bias} (proof in Appendix \ref{appendix:bias}) proves convergence of the expected biased objective $F_B$ to its minimum $F_B^*$ under correlated client participation. Our bound \eqref{eq:convergence_bound} captures the effect of correlation through the factor $\ln{(1/\lambda(\bm{P}))}$: a high correlation worsens the convergence rate. In particular, we found that the numerator of~\eqref{eq:convergence_bound} has a quadratic-over-linear fractional dependence on $\bm{q}$. Minimizing $\epsilon_{\text{opt}}$ leads, in general, to a different choice of $\bm{q}$ than minimizing $\epsilon_{\text{bias}}$.

\subsection{Minimizing the total error \texorpdfstring{$\epsilon \leq {\epsilon}_{\text{opt}} + {\epsilon}_{\text{bias}}$}{}}
Our analysis points out a trade-off between minimizing $\epsilon_{\text{opt}}$ or~$\epsilon_{\text{bias}}$. Our goal is to find the optimal aggregation weights~$\bm{q}^*$ that 
minimize the \el{upper bound on} total error \ar{$\epsilon(\bm{q})$ 
in~\eqref{eq:total_error}:}
\begin{mini}
    {\scriptstyle \bm{q}}{\epsilon_{\text{opt}}(\bm{q}) + \epsilon_{\text{bias}}(\bm{q})}{}{};
    \addConstraint{\bm{q} \geq 0}
    \addConstraint{\textstyle{\norm{\bm{q}}_1 = Q}}.
	\label{opt:total_error}
\end{mini}

In Appendix~\ref{appendix:convexity} we prove that~\eqref{opt:total_error} is a convex optimization problem, which can be solved with the method of Lagrange multipliers. However, the solution is not of practical utility because the constants in~\eqref{eq:total_error} and~\eqref{eq:convergence_bound} (e.g., $L$, $\mu$, $\Gamma$, $C_P$) are in general problem-dependent and difficult to estimate during training. In particular, $\Gamma$ poses particular difficulties as it is defined in terms of the minimizer of the \ar{target} objective $F$, but the FL algorithm generally minimizes the biased function $F_B$. Moreover, the bound in \eqref{eq:total_error}, \ar{similarly to the bound in~\cite{wang2020tackling}}, diverges when setting some $q_k$ equal to $0$, but this is \gn{simply} an artifact of the proof technique.
A result of more practical interest is the following (proof in Appendix~\ref{appendix:total_variation}): 
\begin{theorem}[{An alternative decomposition of the total error}~${\epsilon}$]
\label{theorem:total_variation}
Under the same assumptions of Theorem \ref{theorem:error}, let $\Gamma' \coloneqq \max_k \{ F_k({\bm{w}}_B^*) - F_k^* \}$.
The following result holds:
    \begin{align}
        \epsilon \leq \underbrace{2 \kappa^2 (F_B(\bm{w}) - F_B^*)}_{\coloneqq \epsilon_{\text{opt}}} 
        + \underbrace{8 \kappa^4 d_{TV}^{2}(\bm{\alpha}, \bm{p}) \Gamma'}_{\coloneqq \epsilon_{\text{bias}}'},
        \label{eq:total_variation}
    \end{align}
where $d_{TV}(\bm{\alpha}, \bm{p}) \coloneqq \frac{1}{2} \sum_{k=1}^{N} \abs{\alpha_k - p_k}$ is the total variation distance between the probability distributions $\bm{\alpha}$ and $\bm{p}$. 
\end{theorem}

The new constant $\Gamma'$ is defined in terms of $\bm{w}_B^ *$, and then it is easier to evaluate during training. However, $\Gamma'$ depends on~$\bm{q}$, because it is evaluated at the point of minimum of $F_B$. This dependence makes the minimization of the right-hand side of \eqref{eq:total_variation} more challenging (for example, the corresponding problem is not convex). 
We study the minimization of the two terms $\epsilon_{\text{opt}}$ and~$\epsilon'_{\text{bias}}$ separately and learn some insights, which we use to design the new FL algorithm \cafed.

\subsection{Minimizing \texorpdfstring{${\epsilon}_\text{{opt}}$}{}}
\label{subsection:opt}
The minimization of $\epsilon_{\text{opt}}$ is still a convex optimization problem (Appendix~\ref{appendix:opt}).
In particular, at the optimum non-negative weights are set accordingly to  $q_{k}^* = a (\lambda^* \pi_k - \theta^*)$ with $a$, $\lambda^*$, and $\theta^*$ positive constants (see~\eqref{kkt:y}). It follows that 
clients with smaller availability get smaller weights in the aggregation. In particular, this suggests that clients with the smallest availability can be excluded from the aggregation, leading to the following guideline:

\emph{\underline{Guideline A}: to speed up the convergence, \ar{we can exclude, i.e., set $q_k^*=0$}, the clients with lowest availability $\pi_k$}.


\ar{This guideline can be justified intuitively: updates from clients with low participation may be too sporadic to allow the FL algorithm to keep track of their local objectives. They act as a noise slowing down the algorithm's convergence. It may be advantageous to exclude these clients from participating.}

We observe that the choice of the aggregation weights $\bm{q}$ does not affect the clients' availability process and, in particular, $\lambda(\bm{P})$. However, if the algorithm excludes some clients, it is possible to consider the state space of the Markov chain that only specifies the availability state of the remaining clients, and this Markov chain may have different spectral properties. For the sake of concreteness, we consider here (and in the rest of the paper) the particular case when the availability of each client $k$ evolves according to a two-states Markov chain $(A_t^{k})_{t \geq 0}$ with transition probability matrix $\bm{P_k}$ and these Markov chains are all independent. In this case, the aggregate process is described by the product Markov chain $(A_t)_{t \geq 0}$ \ar{with transition matrix $\bm{P} = \bigotimes_{\scriptstyle k \in \mathcal K} \bm{P_k}$ and $\lambda(\bm{P})= \max_{\scriptstyle k \in \mathcal K} \lambda(\bm{P_k})$, where $\bm{P_i} \bigotimes \bm{P_j}$ denotes the Kronecker product between matrices $\bm{P_i}$ and $\bm{P_j}$~\cite[Exercise~12.6]{levin2017markov}}.
In this setting, it is possible to redefine the Markov chain $(A_t)_{t \geq 0}$ by taking into account the \ar{reduced state space defined by the} clients with a non-null aggregation weight, i.e., \ar{$\bm{P'} = \bigotimes_{\scriptstyle k' \in \mathcal K | q_{k'} >0} \bm{P_{k'}}$ and $\lambda(\bm{P'})= \max_{\scriptstyle k' \in \mathcal K | q_{k'} >0} \lambda(\bm{P_{k'}})$}, which is potentially smaller than the case when all clients participate to the aggregation. These considerations lead to the following guideline:

\emph{\underline{Guideline B}: to speed up the convergence, \ar{we can exclude,
i.e., set $q_k^*=0$}, the clients with largest $\lambda(\bm{P_k})$}. 

Intuition also supports this guideline.
Clients with large $\lambda(\bm{P_k})$ tend to be available or unavailable for long periods of time. Due to the well-known catastrophic forgetting problem affecting gradient methods~\cite{goodfellow2013empirical,kemker2018measuring}, these clients may unfairly steer the algorithm toward their local objective when they appear at the final stages of the training period. Moreover, their participation in the early stages may be useless, as their contribution will be forgotten during their long absence. The FL algorithm may benefit from directly neglecting such clients.

We observe that guideline~B strictly applies to this specific setting where clients' dynamics are independent (and there is no spatial correlation). We do not provide a corresponding guideline for the case when clients are spatially correlated (we leave this task for future research). However, in this more general setting, it is possible to ignore guideline~B but still draw on guidelines~A and~C, or still consider guideline~B if clients are \arx{spatially correlated (see discussion in Section~\ref{sec:discussion_spatial})}.

\subsection{Minimizing \texorpdfstring{${\epsilon}_\text{{bias}}'$}{}}
The bias error $\epsilon'_{\text{bias}}$ in~ \eqref{eq:total_variation} vanishes when the total variation distance between the \ar{target importance $\bm{\alpha}$ and the biased importance $\bm{p}$} is zero, i.e., when \ar{$q_k \propto \alpha_k/\pi_k, \forall k \in \mathcal{K}$}. 
Then, after excluding the clients that contribute the most to the optimization error \ar{and particularly slow down the convergence (guidelines A and B)}, we can assign to the remaining clients an aggregation weight inversely proportional to their availability, \ar{such that the bias error ${\epsilon}_\text{{bias}}'$ is minimized}.

\emph{\underline{Guideline C}: to reduce the bias error, 
we set $q_k^* \propto \alpha_k / \pi_k$
for the clients \el{that are} not excluded by the previous guidelines}.

%% file: algorithm.tex
\IncMargin{1em}
\begin{algorithm}[t]
    \SetKwFunction{Protocol}{get}
    \SetKwInOut{Input}{Input}
    \SetKwInOut{Output}{Output}
    \SetAlgoLined
    \SetKwProg{Fn}{Function}{:}{}
    \Input{
    $\bm{w}_{0, 0}$, 
    $\bm{\alpha}$, 
    $\bm{q}^{\scriptscriptstyle (0)}$,
    $\{\eta_{t}\}_{\scriptscriptstyle t=1}^{\scriptscriptstyle T}$, 
    $\eta_{s}$, 
    $E$, 
    $\beta$, 
    $\tau$
    }
    Initialize $\hat{\bm{F}}^{\scriptscriptstyle (0)}$, \ar{$\hat{\bm{F}}^*$}, $\hat{\Gamma}^{'\scriptscriptstyle (0)}$, $\hat{\bm{\pi}}^{\scriptscriptstyle (0)}$, and $\hat{\bm{\lambda}}^{\scriptscriptstyle (0)}$\;
    \For{$t=1, \dots, T$}{
        Receive set of active client $A_{t}$,  
        loss vector $\bm{F}^{\scriptscriptstyle (t)}$\; 
        \label{line:losses}
        Update 
        $\hat{\bm{F}}^{\scriptscriptstyle (t)}$, $\hat{\Gamma}^{'\scriptscriptstyle (t)}$, $\hat{\bm{\pi}}^{\scriptscriptstyle (t)}$, 
        and 
        $\hat{\bm{\lambda}}^{\scriptscriptstyle (t)}$\; \label{line:estimate_params} 
        \ar{Initialize $\bm{q}^{\scriptscriptstyle (t)} = \frac{\bm{\alpha}}{\hat{\bm{\pi}}^{\scriptscriptstyle (t)}}$}\;  
        $\bm{q}^{\scriptscriptstyle (t)} \gets$ 
        \Protocol{$
        \bm{q}^{\scriptscriptstyle (t)}, 
        \bm{\alpha}, 
        \hat{\bm{F}}^{\scriptscriptstyle (t)}, \ar{\hat{\bm{F}}^*}, \hat{\Gamma}^{'\scriptscriptstyle (t)}, \hat{\bm{\pi}}^{\scriptscriptstyle (t)}, \hat{\bm{\lambda}}^{\scriptscriptstyle (t)}$}\; 
        \label{line:lambda}
        $\bm{q}^{\scriptscriptstyle (t)} \gets $ 
        \Protocol{$
        \bm{q}^{\scriptscriptstyle (t)}, 
        \bm{\alpha},  
        \hat{\bm{F}}^{\scriptscriptstyle (t)}, \ar{\hat{\bm{F}}^*}, \hat{\Gamma}^{'\scriptscriptstyle (t)}, \hat{\bm{\pi}}^{\scriptscriptstyle (t)}, \shortminus \hat{\bm{\pi}}^{\scriptscriptstyle (t)}$}\; 
        \label{line:pi}
        \For{client $\{k \in A_{t}; ~q_{k}^{(t)} > 0\}$, in parallel}{
            \For{$j=0,\dots, E-1$}{
                \label{line:local_step_1}
                $\bm{w}_{t,j+1}^k = \bm{w}_{t,j}^k - \eta_t \nabla F_k (\bm{w}_{t,j}^k, \mathcal{B}_{t,j}^k)$ \; 
                \label{line:local_step_2}
            }
            $\Delta_{t}^{k} \gets \bm{w}_{t, E} - \bm{w}_{t, 0}$\;
        } 
        $\bm{w}_{t+1,0} \gets \proj{W}{\bm{w}_{t,0} + \eta_{s} \sum_{k\in A_{t}}q_{k}^{\scriptscriptstyle (t)}\cdot \Delta_{t}^{k}}$\;  \label{line:global_step}
    }
    \BlankLine
    \Fn{
    \Protocol{$\bm{q}$, 
              $\bm{\alpha}$, 
              $\bm{F}$, 
              \ar{$\bm{F}^*$},
              $\Gamma$, 
              $\bm{\pi}$, 
              $\bm{\rho}$}
    }
    {   
        \label{line:exploration_2}
        $\ar{\mathcal{K}} \gets$ sort by descending order in $\bm{\rho}$\; \label{line:sort}
        $\ar{\hat{\epsilon}} \gets \langle \bm{F} \ar{- \bm{F}^*},  \bm{\pi} \tilde{\odot} \bm{q} \rangle + d_{TV}^{2}(\bm{\alpha}, \bm{\pi} \tilde{\odot} \bm{q})\cdot \Gamma$\;
        \For{$k \in \mathcal{K}$}{
            $q^{+}_{k} \gets 0$\; \label{line:remove}
            $\ar{\hat{\epsilon}^{+}} \gets \langle \bm{F} \ar{- \bm{F}^*},  \bm{\pi} \tilde{\odot} \bm{q}^{+} \rangle + d_{TV}^{2}(\bm{\alpha}, \bm{\pi} \tilde{\odot} \bm{q}^{+})\cdot \Gamma$\; \label{line:improvement}
            \uIf{$\hat{\epsilon} - \hat{\epsilon}^{+} \geq \tau$}{ \label{line:stop}
                $\hat{\epsilon} \gets \hat{\epsilon}^{+}$\;
                $\bm{q} \gets \bm{q}^{+}$\;
            }
        }
        \Return{$\bm{q}$}
    }
    \caption{\cafed{} (Correlation-Aware FL)}
    \label{alg:clients_sampling}
\end{algorithm}
Guidelines~A and~B in Section~\ref{sec:analysis} suggest that \el{ the minimization of}~$\epsilon_{\text{opt}}$ can lead to \el{the exclusion of some} available clients from the aggregation step \eqref{eq:fedavg_global_aggregation}, in particular those with low availability and/or high correlation.
For the remaining clients, guideline~C proposes to set their aggregation weight inversely proportional to their availability to reduce the bias error~\ar{$\epsilon_{\text{bias}}'$}. Motivated by these insights, we propose \cafed, \ar{a client sampling and aggregation strategy that takes into account the problem of correlated client availability in FL}, described in Algorithm~\ref{alg:clients_sampling}. \cafed{} learns during training which \el{are the} clients to exclude and how to set the aggregation weights of the other clients to achieve a good trade-off between~$\epsilon_{\text{opt}}$ and~\ar{$\epsilon_{\text{bias}}'$}.
While guidelines~A and~B \ar{indicate which} clients to remove, the \el{exact} number of clients to remove at round~$t$ is identified by \gn{minimizing $\epsilon^{(t)}$ as a proxy for the bound in~\eqref{eq:total_variation}:}\footnote{
Following~\eqref{eq:total_variation}, one could reasonably introduce a hyper-parameter to weigh
the relative importance of the optimization and bias terms in the sum. \arx{We discuss this additional optimization of \cafed{} in Section~\ref{sec:discussion_kappa}.}} 
\begin{align}
\ar{{\epsilon}^{(t)}} \coloneqq F_B(\bm{w}_{t,0}) \ar{- F_B^*} + d_{TV}^{2}(\bm{\alpha}, \bm{p}) \Gamma'.
\label{eq:proxy}
\end{align}


\subsection{\cafed's core steps}

At each communication round $t$, the server sends the current model $\bm{w}_{t,0}$ to all active clients and each client $k$ sends back a noisy estimate $F^{\scriptscriptstyle (t)}_k$ of the current loss computed on a batch of samples $\mathcal{B}_{t,0}^k$, i.e., $F^{\scriptscriptstyle (t)}_k=\frac{1}{|\mathcal{B}_{t,0}^k|} \sum_{\xi \in \mathcal{B}_{t,0}^k}f(\bm{w}_{t,0},\xi)$ (line~\ref{line:losses}). The server uses these values and the information about the current set of available clients $A_t$ to refine its own estimates of each client's loss \ar{($\hat{\bm{F}}^{\scriptscriptstyle (t)} = (\hat F_k^{\scriptscriptstyle (t)})_{k \in \mathcal{K}}$), and each client's loss minimum value ($\hat{\bm{F}}^* = (\hat F_k^*)_{k \in \mathcal{K}}$)}, as well as of~$\Gamma'$, $\pi_k$, $\lambda_k$, \ar{and~$\epsilon^{\scriptscriptstyle (t)}$}, denoted as \gn{$\hat{{\Gamma}}^{'\scriptscriptstyle(t)}$}, $\hat{\pi}_k^{\scriptscriptstyle (t)}$, $\hat{\lambda}_k^{\scriptscriptstyle (t)}$, \ar{and~$\hat \epsilon^{\scriptscriptstyle (t)}$}, respectively (possible estimators are described below) (line~\ref{line:estimate_params}).

The server decides whether excluding clients whose availability \ar{pattern} exhibits high correlation \ar{(high~$\hat \lambda_k^{\scriptscriptstyle (t)}$)} (line~\ref{line:lambda}). \ar{First,} the server considers all clients in descending order of $\hat{\bm{\lambda}}^{\scriptscriptstyle (t)}$ (line~\ref{line:sort}), and evaluates if, by excluding them (line~\ref{line:remove}), \ar{$\hat{\epsilon}^{\scriptscriptstyle (t)}$}~appears to be \ar{decreasing} by more than a threshold~$\tau \geq 0$ (line~\ref{line:stop}). \ar{Then, the server considers clients in ascending order of $\hat{\bm{\pi}}^{\scriptscriptstyle (t)}$, and} repeats the same procedure to \ar{possibly} exclude some of the clients with low availability \ar{(low $\hat \pi_k^{\scriptscriptstyle (t)}$)} (lines~\ref{line:pi}).


Once the participating clients (those with $q_k>0$) have been selected, the server notifies them to proceed updating the current models (lines~\ref{line:local_step_1}--\ref{line:local_step_2}) according to~\eqref{eq:localSGD}, while the other available clients stay idle. Finally, model's updates are aggregated according to \eqref{eq:fedavg_global_aggregation} (line~\ref{line:global_step}).

\subsection{Estimators}
We now briefly discuss possible implementation of the estimators $\hat{F}_k^{\scriptscriptstyle(t)}$, \ar{$\hat{F}_k^*$}, \ar{$\hat{{\Gamma}}^{'\scriptscriptstyle(t)}$},  $\hat{\pi}_k^{\scriptscriptstyle(t)}$, and $\hat{\lambda}_k^{\scriptscriptstyle(t)}$. Server's estimates for the clients' \ar{local} losses \ar{($\hat{\bm{F}}^{\scriptscriptstyle (t)} = (\hat{F}_k^{\scriptscriptstyle (t)} )_{k \in \mathcal{K}}$)} can be obtained \ar{from the received active clients' losses ($\bm{F}^{\scriptscriptstyle (t)} = (F_k^{\scriptscriptstyle (t)})_{k \in A_t}$)} through an auto-regressive filter with parameter $\beta \in (0,1]$:
\begin{align}
    \hat{\bm{F}}^{\scriptscriptstyle (t)} = (\mathbf{1} - \beta \mathds{1}_{A_t}) \odot  \hat{\bm{F}}^{(t-1)} + \beta \mathds{1}_{A_t} \odot \bm{F}^{\scriptscriptstyle (t)},
\end{align}
where \ar{$\odot$ denotes the component-wise multiplication between vectors}, and $\mathds{1}_{A_t}$ is a \gn{$N$-dimensions binary vector whose $k$-th component equals $1$} if and only if $k$ is active at round $t$, i.e., $k \in A_t$. \ar{The server can keep track of the clients' loss minimum values and estimate $F_k^*$ as $\hat{F}_k^* = \min_{\scriptscriptstyle s \in [0, t]} \hat{F}_{k}^{\scriptscriptstyle (s)}$.} The values of $F_B(\bm{w}_{t,0})$, \ar{ $F_B^*$}, $\Gamma'$, \ar{and $\epsilon^{\scriptscriptstyle (t)}$} can be estimated as follows:
\begin{gather}
        \textstyle
        \hat{F}_{B}^{\scriptscriptstyle (t)} - \ar{\hat{F}_{B}^* }
        = \langle \hat{\bm{F}}^{\scriptscriptstyle (t)} - \ar{\hat{\bm{F}}^*},  \hat{\bm{\pi}}^{\scriptscriptstyle (t)} \tilde{\odot} \bm{q}^{\scriptscriptstyle (t)} \rangle, 
        \label{eq:update_biased_objective_estimate}
        \\
        \textstyle
        \ar{\hat{\Gamma}^{'\scriptscriptstyle (t)}} 
        = \max_{\ar{k \in \mathcal{K}}} (\hat{F}_{k}^{\scriptscriptstyle (t)} - \ar{\hat{F}_k^*}),
        \label{eq:update_gamma_estimate}
        \\
        \textstyle
        \ar{\hat{\epsilon}^{\scriptscriptstyle (t)}
        = \hat{F}_{B}^{\scriptscriptstyle (t)} - \ar{\hat{F}_{B}^*} + d_{TV}^2(\bm{\alpha},\hat{\bm{\pi}}^{\scriptscriptstyle (t)} \tilde{\odot} \bm{q}^{\scriptscriptstyle (t)}) \cdot \ar{\hat{\Gamma}^{'\scriptscriptstyle (t)}}}.
        \label{eq:update_error_estimate}
    \end{gather}
where $\bm{\pi} \tilde{\odot} \bm{q} \in \mathbb{R}^{N}$, such that $\left(\bm{\pi} \tilde{\odot} \bm{q}\right)_{k} = \frac{\pi_{k} q_{k}}{\sum_{h=1}^N \pi_{h} q_{h}},~k\in \mathcal{K}$.

For $\hat \pi_k^{\scriptscriptstyle (t)}$, the server can simply keep track of the total number of times client $k$ was available up to time~$t$ and compute $\hat \pi_k^{\scriptscriptstyle (t)}$ using a Bayesian estimator with beta prior, i.e., $\hat \pi_k^{\scriptscriptstyle (t)} = (\sum_{s\le t}\mathds{1}_{k \in A_s}+n_k)/(t+n_k+m_k)$, where $n_k$ and $m_k$ are the initial parameters of the beta prior.

For $\hat \lambda_k^{\scriptscriptstyle (t)}$, the server can assume the client's availability evolves according to a Markov chain with two states (available and unavailable), track the corresponding number of state transitions, and estimate the transition matrix \ar{$\hat{\bm{P}}_k^{\scriptscriptstyle (t)}$} through a Bayesian estimator similarly to what done for $\hat \pi_k^{\scriptscriptstyle (t)}$. Finally, $\hat \lambda_k^{\scriptscriptstyle (t)}$ is obtained computing the eigenvalues of $\hat{\bm{P}}_k^{\scriptscriptstyle (t)}$.

\subsection{\cafed's computation/communication cost}
\cafed{} aims to improve training convergence and not to reduce its computation and communication overhead. Nevertheless, excluding some available clients reduces the overall training cost, as we will discuss in this section referring, for the sake of concreteness, to neural networks' training.

The available clients \ff{not selected for \ar{training} \el{are \ar{only} requested to} \ar{evaluate their} local loss} on the current model once on a single batch \el{instead than} \ar{performing} $E$ gradient updates, which would require roughly \ar{$2 \times E - 1$} more calculations (because of the forward and backward pass). For the selected clients, there is no extra computation cost as computing the loss corresponds to the forward pass they should, in any case, perform during the first local gradient update.

In terms of communication, the excluded clients only transmit the loss, a single scalar, much smaller than the model update. Conversely, participating clients transmit the local loss and the model update. Still, this additional overhead is negligible and likely fully compensated by the communication savings for the excluded clients.






%% file: experiments.tex
\begin{figure*}[t]
    \centering
    \includegraphics[width=0.8\textwidth, keepaspectratio]{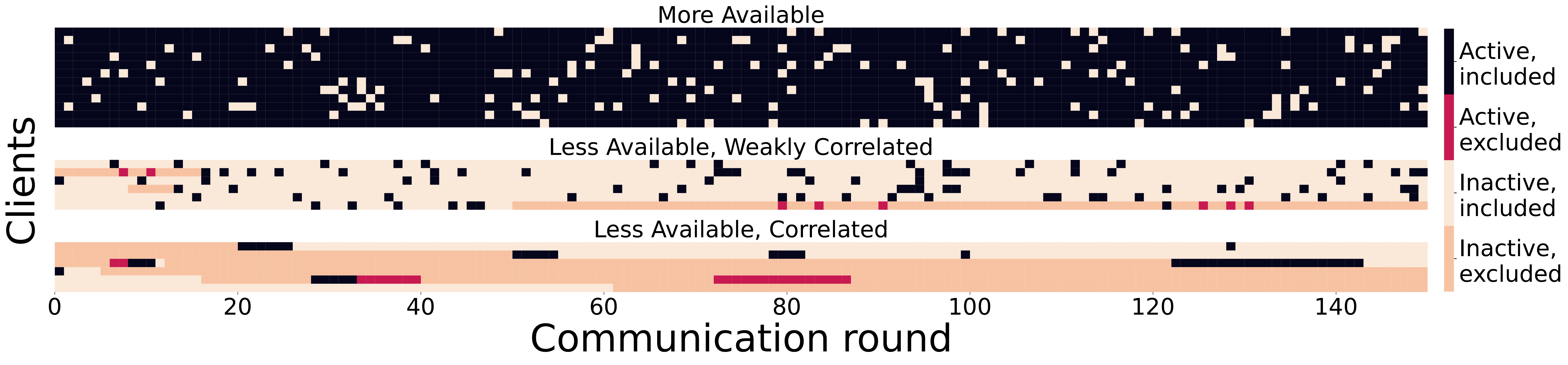}
    \caption[]
    {\small \ar{Clients' activities and \cafed{}'s clients selection on the synthetic dataset.} }
    \label{f:heatmap}
\end{figure*}
\begin{figure}[t]
    \centering
    \includegraphics[width=0.24\textwidth, keepaspectratio]{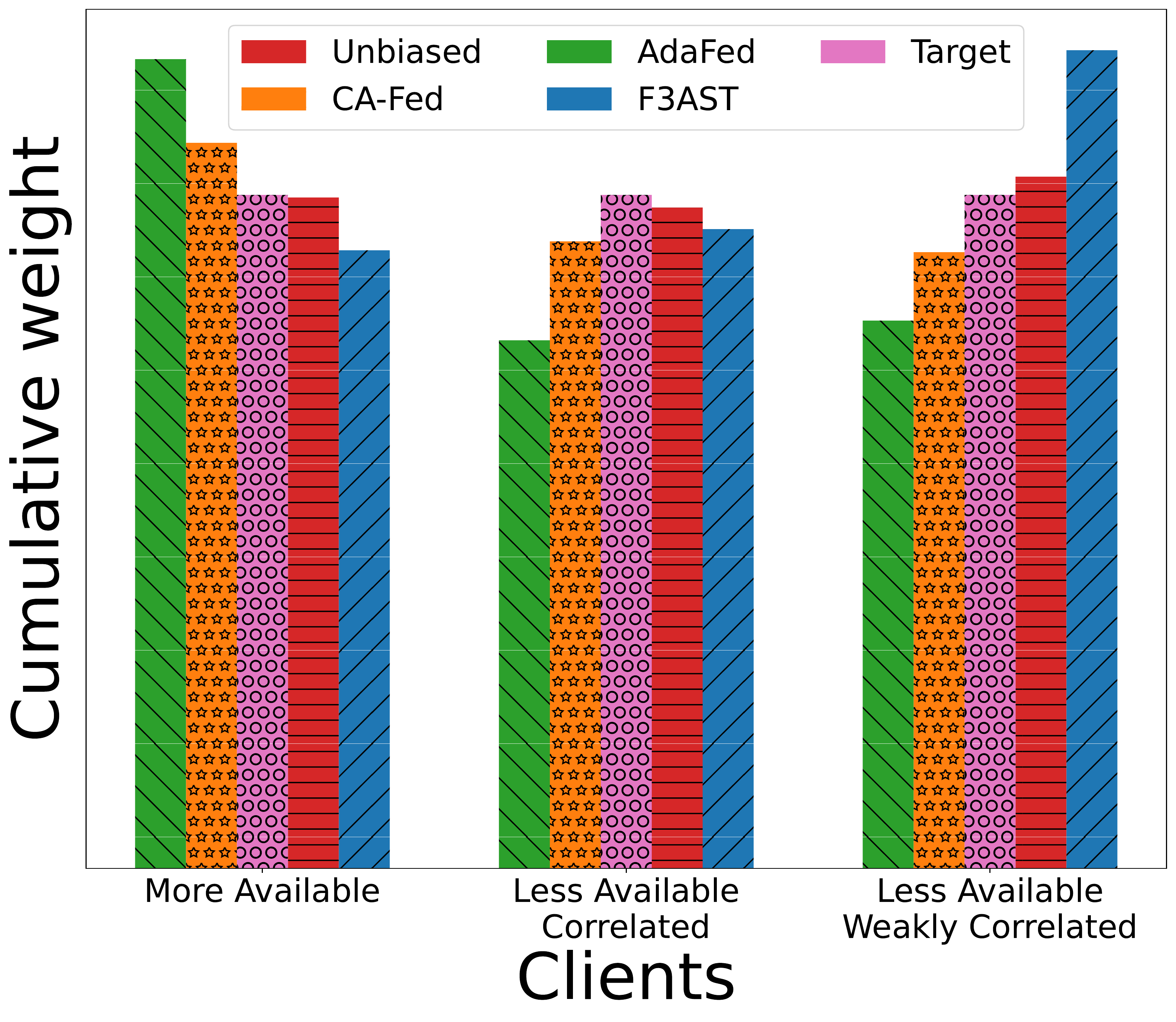}
    \caption[]
    {\small \ar{Importance given to the clients by the different algorithms throughout a whole training process on the synthetic dataset.}}
    \label{f:history}
\end{figure}

\subsection{Experimental Setup}

\paragraph{Federated system simulator}
In our experiments, we simulate the clients' availability dynamics featuring different levels of temporal correlations. We model the activity of each client as a two-state homogeneous Markov process \ar{with state space $\mathcal{S} = \{ \text{``\emph{active}''}, \text{``\emph{inactive}''} \}$. We use $p_{k,s}$ to denote the probability that client $k \in \mathcal{K}$ remains in state $s \in \mathcal{S}$}. 

In order to simulate the statistical heterogeneity present in the federated learning system, we consider an experimental setting with two disjoint groups of clients ${G}_{i},~i=1,2$, to which we associate two different data distributions $\mathcal{P}_{i},~i=1,2$, to be precised later. Let~$r_{i}=|{G}_{i}|/N,~i=1, 2$ denote the fraction of clients in group $i=1, 2$. In order to simulate the heterogeneity of clients' availability patterns in realistic federated systems, we split the clients of each group in two classes uniformly at random: ``more available'' clients whose steady-state probability to be active is~$\ar{\pi_{k,\text{active}}} = 1/2 + g$  and ``less available'' clients with  $\ar{\pi_{k, \text{active}}} = 1/2 - g$, where $g\in(0, 1/2)$ is a parameter controlling the  heterogeneity of clients availability. 
We furthermore split each class of clients in two sub-classes uniformly at random: ``correlated'' clients that tend to persist in the same state ($\ar{\lambda_k} =\nu$ with values of $\nu$ close to $1$), 
and ``weakly correlated'' clients that are almost as likely to keep as to change their state ($\ar{\lambda_k} \sim \mathcal{N}(0, \varepsilon^{2})$, with $\varepsilon$ close to $0$).
In our experiments, we suppose that $r_{1}=r_{2}=1/2$, $g=0.4$, $\nu=0.9$, and $\varepsilon=10^{-2}$.  

\paragraph{Datasets and models} 
All experiments are performed on a binary classification synthetic dataset (described in Appendix~\ref{appendix:synthetic}) and on the real-world MNIST dataset~\cite{lecun-mnisthandwrittendigit-2010}, using $N=24$ clients.
For MNIST dataset, we introduce statistical heterogeneity across the two groups of clients (i.e., we make the two distributions $\mathcal P_1$ and $\mathcal P_2$ different), following the same approach  in~\cite{sattler2020clustered}: 1)~every client is assigned a random subset of the total training data; 2)~the data of clients from the second group is  modified  by  randomly  swapping two pairs of labels. We maintain the  original  training/test  data split of MNIST and use $20\%$ of  the training dataset as validation dataset. 
We use a linear classifier with a ridge penalization of parameter $10^{-2}$, which is a strongly convex objective function, for both the synthetic and the real-world MNIST datasets. 

\paragraph{Benchmarks} We compare \cafed{}, defined in Algorithm~\ref{alg:clients_sampling}, with the \unbiased{} \ar{aggregation strategy}, where all the active clients participate and receive a weight inversely proportional to their availability, and with the state-of-the-art FL algorithms 
discussed in Section~\ref{sec:problem}: \fast~\cite{ribero2022federated} and \adafed~\cite{tan2022adafed}. 
\ar{We tuned the learning rates $\eta,~\eta_{s}$ via grid search, on the grid $\eta : \{ 10^{-3}, 10^{-2.5}, 10^{-2}, 10^{-1.5}, 10^{-1} \}$, $\eta_{s} : \{ 10^{-2}, 10^{-1.5}, 10^{-1}, 10^{-0.5}, 10^{0} \}$.} For \cafed{}, we used $\tau=0$, $\beta=0.2$. We assume all algorithms \ar{can access} an oracle providing the true availability parameters for each client. In practice, \unbiased, \adafed, and \fast{} rely on the exact knowledge of $\ar{\pi_{k,\text{active}}}$, and \cafed{} on $\ar{\pi_{k,\text{active}}}$ and $\ar{\lambda_k}$. \footnote{
\ar{The authors have provided public access to their code and data at: \\ \hspace*{1.1em} \href{https://github.com/arodio/CA-Fed}{https://github.com/arodio/CA-Fed}.}
}

\begin{figure}[t]
    \centering
    \begin{subfigure}[b]{0.235\textwidth} 
        \centering 
        \includegraphics[width=\textwidth, keepaspectratio]{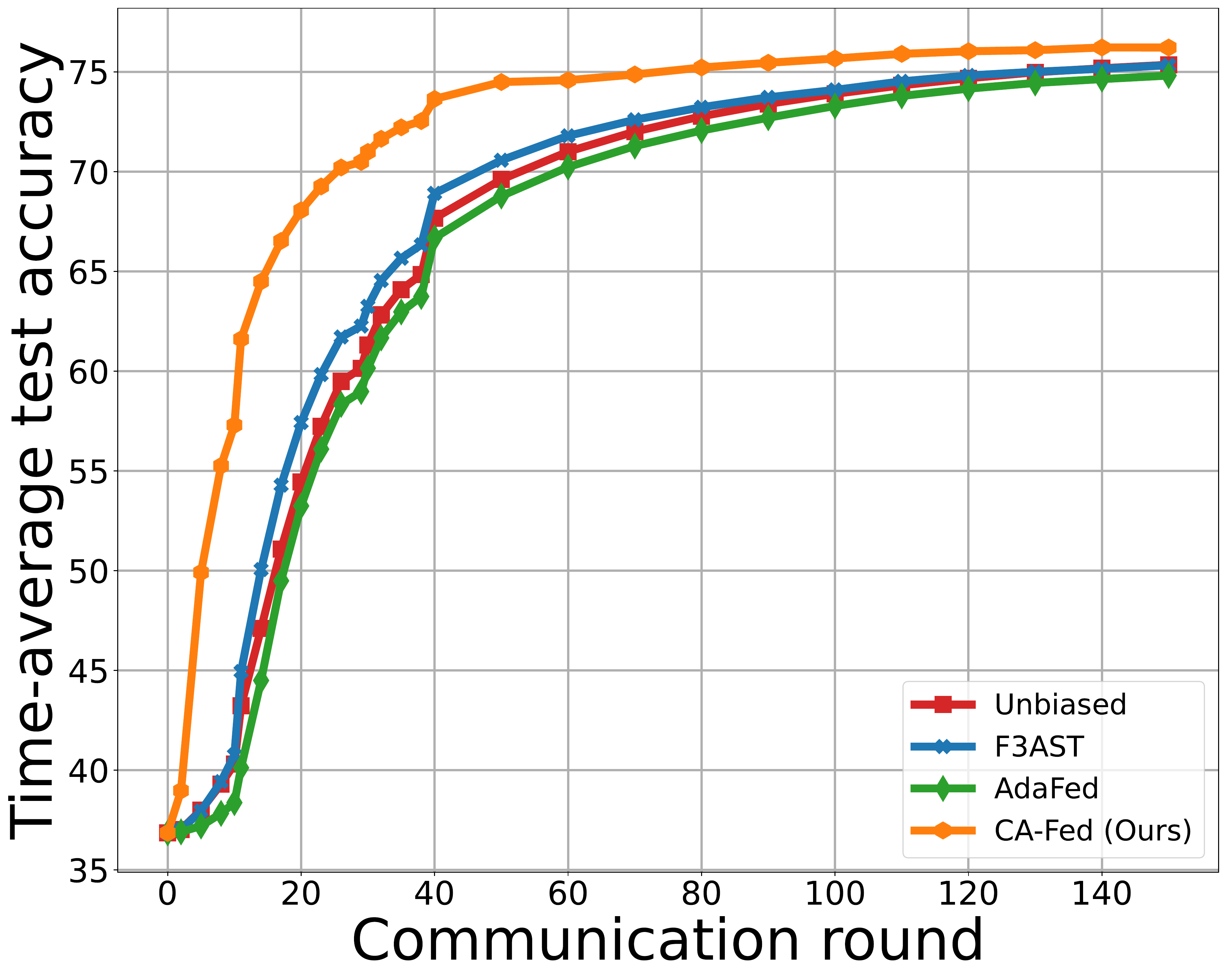}
        \subcaption[]{\small Synthetic}
        \label{f:smooth_test_acc_syntehtic}
    \end{subfigure}
    \hfill
    \begin{subfigure}[b]{0.235\textwidth}
        \centering
        \includegraphics[width=\textwidth, keepaspectratio]{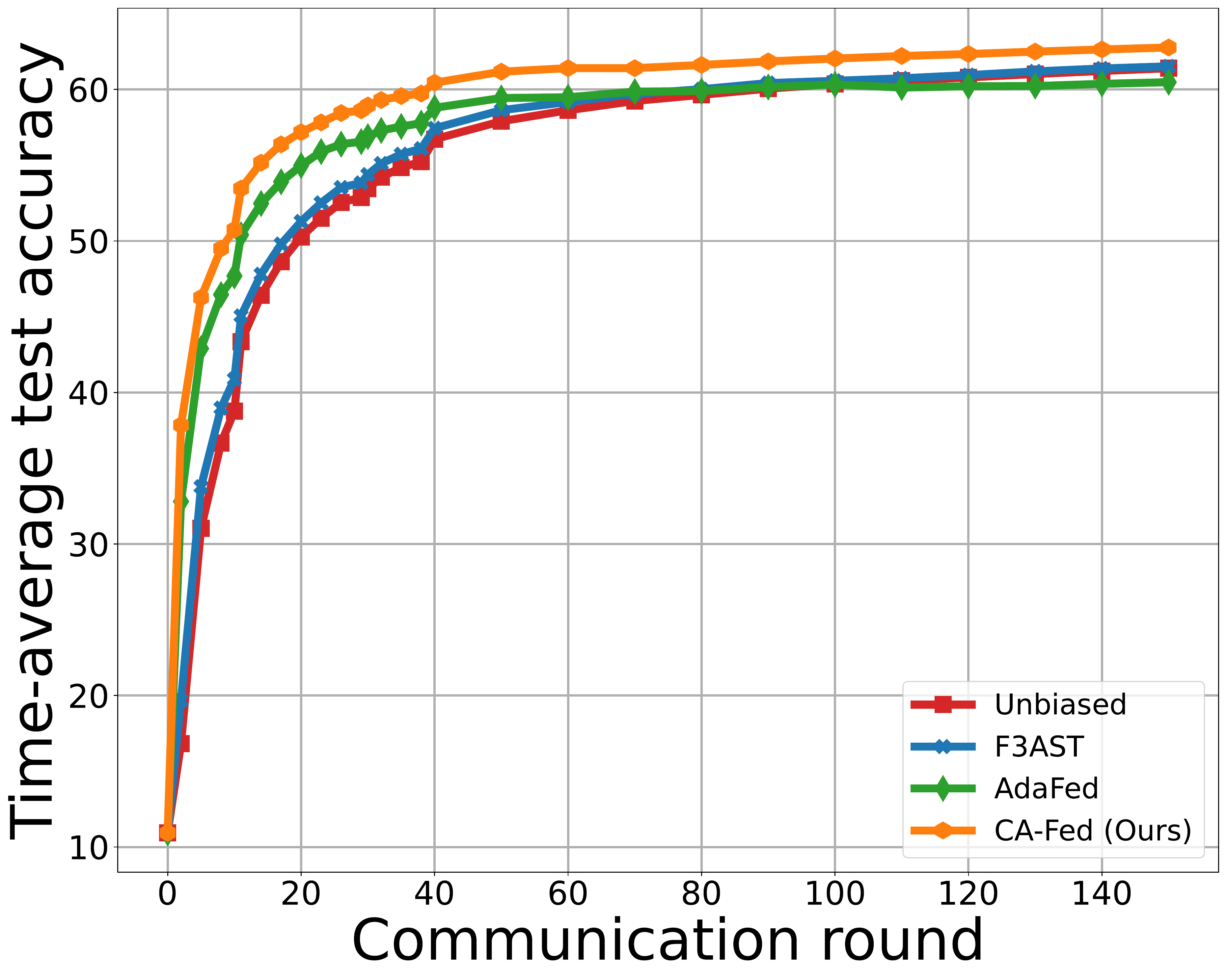}
        \subcaption[]{\small MNIST}
        \label{f:smooth_test_acc_mnist}
    \end{subfigure}
    \caption[]
    {\small \ar{Test accuracy vs number of communication rounds.}} 
    \label{f:test_acc}
\end{figure}

\subsection{Experimental Results}

Figure~\ref{f:heatmap} shows the availability of each client during a training run on the synthetic dataset. Clients selected (resp.~excluded) by \cafed{} are highlighted in \ar{black (resp.~red)}. We observe that  excluded clients tend to be those with low average availability or high correlation.

Figure~\ref{f:history} shows the \ar{importance} $p_k$ (averaged over time) given by different algorithms to each client $k$ during \ar{a full} training run. We observe that \ar{all the algorithms, except \unbiased{}, depart from the target importance} $\bm{\alpha}$. As suggested by guidelines~A and~B, \cafed{} tends to favor the group of ``more available'' clients, at the expense of the ``less available'' clients.

Figure~\ref{f:test_acc} shows the time-average accuracy up to round~$t$ of the learned model averaged over \ar{three} different runs. On both datasets, \cafed{} achieves the highest accuracy, which is about a percentage point higher than the second best algorithm \ar{(\fast{})}.
Table~\ref{tab:convergence_rate} shows for each algorithm: the average over three runs of the maximum test accuracy achieved during training, the time-average test accuracy \ar{achieved during training,} together with its standard deviation within the second half of the training period. Results show that while \cafed{} achieves a
maximum accuracy which is comparable to the \unbiased{} baseline and state-of-the-art \adafed{} and \fast{}, it gets a higher time-average accuracy \ar{($1.24$ percentage points)} in comparison to the second best \ar{(\fast{})}, and a smaller standard deviation \ar{($1.5 \times$)} in comparison to the second best~\ar{(\fast{})}.

\begin{table}[t]
    \caption{\small \ar{Maximum and time-average test accuracy, together with their standard deviations, on the Synthetic / MNIST datasets.}}
    \centering
    \begin{center}
    \begin{small}
    \begin{sc}
    \resizebox{0.48\textwidth}{!}{
        \begin{tabular}{l c c c}
            \toprule
            & \multicolumn{3}{c}{Test Accuracy} 
            \\
            & Maximum & Time-average & Standard deviation 
            \\
            \midrule
            \unbiased{} & \ar{$78.94$} $\,/\,$ \ar{$64.87$} & \ar{$75.32$} $\,/\,$ \ar{$61.39$} & \ar{$0.48$} $\,/\,$ \ar{$1.09$}
            \\
            \fast{} & \ar{$78.97$} $\,/\,$ \ar{$64.91$} & \ar{$75.33$} $\,/\,$ \ar{$61.52$} & \ar{$0.40$} $\,/\,$ \ar{$0.94$}
            \\
            \adafed{} & \ar{$78.69$} $\,/\,$ \ar{$63.77$} & \ar{$74.81$} $\,/\,$ \ar{$60.48$} & \ar{$0.59$} $\,/\,$ \ar{$1.37$}
            \\
            \cafed{} & \ar{$\bm{79.03}$} $\,/\,$ \ar{$\bm{64.94}$} & \ar{$\bm{76.22}$} $\,/\,$ \ar{$\bm{62.76}$} & \ar{$\bm{0.28}$} $\,/\,$ \ar{$\bm{0.61}$}
            \\
            \bottomrule
        \end{tabular}
    }
    \end{sc}
    \end{small}
    \end{center}
    \label{tab:convergence_rate}
\end{table}

%% file: discussion.tex
In this section, we discuss some general concerns and remarks on our algorithm.

\subsection{Controlling the number of excluded clients}
\label{sec:discussion_kappa}
Theorems~\ref{theorem:error} and~\ref{theorem:total_variation} suggest that the condition number $\kappa^2$ can play a meaningful role in the minimization of the total error $\epsilon$. Our algorithm uses a proxy ($\epsilon^{(t)}$) of the total error. To take into account the effect of $\kappa^2$, we can introduce a hyper-parameter that weights the relative importance of the optimization and bias error in~\eqref{eq:proxy}:
\[{\epsilon'}^{(t)} \coloneqq F_B(\bm{w}_{t,0}) - F_B^* + \bar\kappa^2 \cdot d_{TV}^{2}(\bm{\alpha}, \bm{p}) \Gamma'.\]
A small value of $\bar\kappa^2$ penalizes the bias term in favor of the optimization error, resulting in a larger number of clients excluded by \cafed. On the other hand, \cafed{} tends to include more clients for a large value of $\bar\kappa^2$. Asymptotically, for $\bar\kappa^2 \rightarrow + \infty$, \cafed{} reduces to the \unbiased{} baseline. To further improve the performance of \cafed{}, a finer tuning of the values of $\bar\kappa^2$ can be performed.

\subsection{\cafed{} in presence of spatial correlation}
\label{sec:discussion_spatial}
Although \cafed{} is mainly designed to handle temporal correlation, it does not necessarily perform poorly in presence of spatial correlation, as well. 

Consider the following spatially-correlated scenario: clients are grouped in clusters, 
each cluster $c \in \mathcal{C}$ is characterized by an underlying Markov chain, which determines when all clients in the cluster are available/unavailable, the Markov chains of different clusters are independent.
Let $\lambda_c$ denote the second largest eigenvalue in module of cluster-$c$'s Markov chain.
In this case, one needs to exclude all clients in the cluster $\bar c = \arg\max_{c \in \mathcal C} \lambda_c $ to reduce the eigenvalue of the  aggregate Markov chain. 

In this setting, \cafed{} would associate similar eigenvalue estimates to all clients in the same cluster, then it would correctly start considering for exclusion the clients in cluster $\bar c$ and potentially remove sequentially all clients in the same cluster.
These considerations suggest that \cafed{} may  still operate correctly even in presence of 
spatial correlation.


\subsection{About \cafed's fairness}

A strategy that excludes clients from the training phase, such as \cafed{}, may naturally raise fairness concerns. The concept of fairness in FL does not have a unified definition in the literature~\cite[Chapter 8]{ludwig2022federated}: fairness goals can be captured by a suitable choice of the target weights in \eqref{opt:target_objective}. For example, per-client fairness can be achieved by setting $\alpha_k$ equal for every client, while per-sample fairness by setting $\alpha_k$ proportional to the local dataset size $|D_k|$. 
If we assume that the global objective in~\eqref{opt:target_objective} indeed reflects also fairness concerns, then \cafed{} is intrinsically fair, in the sense that it guarantees that the performance objective of the learned model is as close as possible to its minimum value.

%% file: conclusion.tex
This paper presented the first convergence analysis for a \fedavg{}-like FL algorithm under heterogeneous and correlated client availability. The analysis quantifies how correlation adversely affects the algorithm's convergence rate and highlights a general bias-versus-convergence-speed trade-off. Guided by the theoretical analysis, we proposed \cafed{}, a new FL algorithm that tries to balance the conflicting goals of maximizing convergence speed and minimizing model bias. 
Our experimental results demonstrate that adaptively excluding clients with high temporal correlation and low availability is an effective approach to handle the heterogeneous and correlated client availability in FL. 

%% file: appendix.tex
\subsection{Proof of Theorem \ref{theorem:error}}
\label{appendix:error}

We bound the optimization error of the \ar{target} objective as the optimization error of the \ar{biased} objective plus a bias term:
\begin{align*}
    \textstyle
    F(\bm{w}) - F^* 
    &\textstyle
    \stackrel{\text{\tiny (a)}}{\leq}
    \frac{1}{2 \mu} \norm{\nabla F(\bm{w})}^2
    \stackrel{\text{\tiny (b)}}{\leq}
    \frac{L^2}{2 \mu} \norm{\bm{w} - \bm{w}^*}^2 \\
    &\textstyle
    \stackrel{\text{\tiny (c)}}{\leq}
    \frac{L^2}{\mu} (\norm{\bm{w} - \bm{w}_B^*}^2 + \norm{\bm{w}_B^* - \bm{w}^*}^2) \\
    &\textstyle
    \stackrel{\text{\tiny (d)}}{\leq}
    \underbrace{
    \textstyle
    \frac{2 L^2}{\mu^2} (F_B(\bm{w}) - F_B^*)}_{\coloneqq \epsilon_{\text{opt}}} 
    + \underbrace{
    \textstyle
    \frac{2 L^2}{\mu^2} (F(\bm{w}_B^*) - F^*)}_{\coloneqq \epsilon_{\text{bias}}},
\end{align*}
\arx{where $(a)$, $(b)$, and $(d)$ follow from the Assumptions~\ref{assumption:smoothness},~\ref{assumption:convexity}, and the inequality $(c)$ follows from $(a+b)^2 \leq 2a^2 + 2b^2$. In particular, $(b)$ requires $\nabla F_k(\bm{w}_k^*)=0$. Theorem~\ref{theorem:bias} further develops the optimization error $\epsilon_{\text{opt}}$. We now expand $\epsilon_{\text{bias}}$:}
\begin{align}
    \textstyle
    \norm{\nabla F(\bm{w}_B^*)} 
    &\textstyle
    \stackrel{\text{\tiny (e)}}{=}
    \norm{\sum_{k=1}^N (\alpha_k - p_k) \nabla F_k(\bm{w}_B^*)} \notag \\
    &\textstyle
    \stackrel{\text{\tiny (f)}}{\leq}
    L \sum_{k=1}^N \abs{\alpha_k - p_k} \norm{\bm{w}_B^* - \bm{w}_k^*} 
    \label{eq:replace_th3} \\
    &\textstyle
    \stackrel{\text{\tiny (g)}}{\leq}
    L \sqrt{\frac{2}{\mu}} \sum_{k=1}^N \frac{\abs{\alpha_k - p_k}}{\sqrt{p_k}} \sqrt{p_k (F_k(\bm{w}_B^*) - F_k^*)}, \notag
\end{align}
where $(e)$ uses $\nabla F_B(\bm{w}_B^*) = 0$; $(f)$ applies first the triangle inequality, then the $L$-smoothness, and $(g)$ follows from the $\mu$-strong convexity. In addition, $(f)$ requires $\nabla F_k(\bm{w}_k^*)=0$. Similarly to~\cite{wang2020tackling}, in $(g)$ we multiply numerator and denominator by $\sqrt{p_k}$. By direct calculations, it follows that:
\begin{align*}
    \textstyle
    \norm{\nabla F(\bm{w}_B^*)}^2
    &\textstyle
    \stackrel{\text{\tiny (h)}}{\leq}
    \frac{2L^2}{\mu} \Big( \sum_{k=1}^N {\frac{\abs{\alpha_k - p_k}}{\sqrt{p_k}}} \sqrt{p_k (F_k(\bm{w}_B^*) - F_k^*)} \Big)^2 \\
    &\textstyle
    \stackrel{\text{\tiny (i)}}{\leq}
    \frac{2L^2}{\mu} \Big( \sum\limits_{\scriptscriptstyle k=1}^{\scriptscriptstyle N} \frac{(\alpha_k - p_k)^2}{p_k} \Big) 
    \Big( \sum\limits_{\scriptscriptstyle k=1}^{\scriptscriptstyle N} p_k (F_k(\bm{w}_B^*) - F_k^*) \Big) \\
    &\textstyle
    \stackrel{\text{\tiny (j)}}{\leq}
    \frac{2L^2}{\mu} \chi^2_{\bm{\alpha} \parallel \bm{p}} \Gamma,
\end{align*}
where $(i)$ uses the Cauchy–Schwarz inequality, and $(j)$ used:
\begin{align*}
    \textstyle
    \sum_{k=1}^{N} p_k (F_k(\bm{w}_B^*) - F_k^*)
    \leq
    \sum_{k=1}^{N} p_k (F_k(\bm{w}^*) - F_k^*) 
    \leq 
    \Gamma.
\end{align*}
Finally, by strong convexity of $F$, we conclude that:
\begin{align*}
    \textstyle
    F(\bm{w}_B^*) - F^* 
    \leq \frac{1}{2 \mu} \norm{\nabla F(\bm{w}_B^*)}^2
    \leq \frac{L^2}{\mu^2} \chi^2_{\bm{\alpha} \parallel \bm{p}} \Gamma.
    \qed
\end{align*}

\subsection{Proof of Theorem \ref{theorem:bias}}
\label{appendix:bias}

\subsubsection{Additional notation}
let $\bm{w}_{t,j}^k$ be the model parameter vector computed by device $k$ at the global round $t$, local iteration $j$. We define: 
\begin{align*}
    \textstyle
    g_t(A_t) = \sum_{ k \in A_t} q_k \sum_{ j=0}^{ E-1} \nabla F_k (\bm{w}_{t,j}^k, \xi_{t,j}^k),
\end{align*}
and $\bar{g}_t(A_t) = \E\nolimits_{\xi|A_t}[g_t(A_t)].$

Following~\eqref{eq:localSGD} and~\eqref{eq:fedavg_global_aggregation}, the update rule of \texttt{CA-Fed} is:
\begin{align}
    \label{eq:update_rule_appendix}
    \bm{w}_{t+1,0} = \proj{W}{\bm{w}_{t,0} - \eta_t g_t(A_t)}.
\end{align}

\subsubsection{Key lemmas and results}
we provide useful lemmas and results to support the proof of the main theorem.

\emph{Proof of Lemma \ref{lem:prop}.}
\ar{The boundedness of $W$ gives a bound on $({\bm{w}}_{t,0})_{t \geq 0}$ based on the update rules in~\eqref{eq:localSGD} and~\eqref{eq:fedavg_global_aggregation}. From the convexity of $\{F_k\}_{k \in \mathcal{K}}$, it follows that:
\begin{align*}
    D \coloneqq \sup_{\bm{w} \in W, k \in \mathcal{K}} \Vert \nabla F_k (\bm{w}) \Vert < +\infty.
\end{align*}
 Items~\eqref{eq:D},~\eqref{eq:H} are directly derived from the previous observation. Item~\eqref{eq:G} follows combining~\eqref{eq:D} and Assumption~\ref{assumption:variance}}:
 \begin{align*}
     \textstyle
     \E \Vert \nabla F_k (\bm{w}, \xi) \Vert^2 
     \leq
     D^2 + \max\limits_{\scriptscriptstyle k \in \mathcal{K}} \{\sigma_k^2\}
     \coloneqq 
     G^2.
     \qed
 \end{align*}

\begin{lemma}[Convergence under heterogeneous client availability]
\label{lemma:li}
Let the local functions $\{F_k\}_{k \in \mathcal{K}}$ be convex, Assumptions~\ref{assumption:smoothness},~\ref{assumption:variance} hold. If $\eta_t \leq \frac{1}{2L(EQ+1)}$, we have:
\begin{align*}
    \textstyle
    \sum_t \eta_t 
    &\textstyle \E[\sum_{k \in A_t} q_k \left( F_k(\bm{w}_{t,0}) - F_k(\bm{w}_B^*) \right)]
    \leq \\
    & \textstyle
    + \frac{2}{E} \norm{\bm{w}_{0,0} - \bm{w}_B^*}^2
    + 2 \sum_{k=1}^N \pi_k q_k^2 \sigma_k^2 \sum_t \eta_t^2 \\
    & \textstyle
    + \frac{2}{3} \sum_{k=1}^N \pi_k q_k (E-1) (2E-1) G^2 \sum_t \eta_t^2 \\
    & \textstyle
    + 2 L (EQ+2) \sum_{k=1}^N \pi_k q_k \Gamma \sum_t \eta_t^2  
    := C_1 < +\infty.
\end{align*}
\end{lemma}
\emph{Proof of Lemma \ref{lemma:li}.}
\begin{align*}
    \norm{\bm{w}_{t+1,0} - \bm{w}_B^*}^2 
    = \norm{\proj{W}{\bm{w}_{t,0} - \eta_t g_t} - \proj{W}{\bm{w}_B^*}}^2  \notag \\
    \leq \norm{\bm{w}_{t,0} - \eta_t g_t - \bm{w}_B^* + \eta_t \bar{g}_t - \eta_t \bar{g}_t}^2 = A_1 + A_2 + A_3,
\end{align*}
where: 
\begin{align*}
    A_1 &= \norm{\bm{w}_{t,0} - \bm{w}_B^* - \eta_t \bar{g}_t}^2, \\
    A_2 &= 2 \eta_t \langle \bm{\bm{w}}_{t,0} - \bm{w}_{B}^{*} - \eta_t \bar{g}_t, \bar{g}_t - g_t \rangle, \\
    A_3 &= \eta_t^2 \norm{g_t - \bar{g}_t}^2. 
\end{align*}

Note $\E[A_2]=0$. We bound $A_1$, $A_3$ using the key steps in~\cite{li2019convergence}:

(1) the variance of $g_t(A_t)$ is bounded if the variance of the stochastic gradients at each device is bounded:

\begin{align*}
    A_3 
    &\textstyle
    = \E_{ \mathcal{B} \mid A_t}
    \norm{g_t - \bar{g}_t}^2 = \\
    &\textstyle
    = \sum_{k \in A_t} q_k^2  \sum_{j=0}^{ E-1} \E_{\mathcal{B} \mid A_t} \norm{\nabla F_k (\bm{w}_{t,j}^k, \xi_{t,j}^k) {-} \nabla F_k (\bm{w}_{t,j}^k)}^2 \\
    &\textstyle
    \leq E \sum_{k \in A_t} q_k^2 \sigma_k^2;
\end{align*}
(2) the distance of the local model $\bm{w}_{t,E}^k$ from the global model $\bm{w}_{t,0}$ is bounded since the expected squared norm of the stochastic gradients is bounded:
\begin{align*}
    &\textstyle
    \E_{\mathcal{B} \mid A_t}
    \sum_{k \in A_t} q_k \sum_{j=0}^{E-1} \norm{\bm{w}_{t,j}^k - \bm{w}_{t,0}}^2 = \\
    &\textstyle
    \quad
    = \E_{\mathcal{B} \mid A_t} \sum_{k \in A_t} q_k \sum_{j=1}^{E-1} \eta_t^2 \norm{\sum_{j'=0}^{j-1} \nabla F_k (\bm{w}_{t,j'}^k, \xi_{t,j'}^k)}^2 \\
    &\textstyle
    \quad
    \leq \eta_t^2 \sum_{k \in A_t} q_k \sum_{j=1}^{E-1} j \sum_{j'=0}^{j-1} \E_{\mathcal{B} \mid A_t} \norm{\nabla F_k (\bm{w}_{t,j'}^k, \xi_{t,j'}^k)}^2 \\
    &\textstyle
    \quad
    \leq \eta_t^2 \sum_{k \in A_t} q_k G^2 \sum_{j=1}^{E-1} j^2 \\
    &\textstyle
    \quad
    = \frac{1}{6} \eta_t^2 \sum_{k \in A_t} q_k E(E-1)(2E-1) G^2.
    \qed
\end{align*}

\begin{lemma}[Optimization error after $\mathcal{J}_t$ steps] 
\label{lemma:sun}
Let Assumptions~\ref{assumption:markov_chain},~\ref{assumption:space} hold, the local functions $\{F_k\}_{k \in \mathcal{K}}$ be convex, $D, H$ be defined as in~\eqref{eq:D},~\eqref{eq:H}, and $\mathcal{J}_t$ defined as in Theorem~\ref{theorem:bias}. Then:
\begin{align*}
    &\textstyle
    \sum_t
    \textstyle \eta_t 
    \E[\sum_{k \in A_t} q_k (F_{k}(\bm{w}_{t-\mathcal{J}_t,0}) - F_{k}(\bm{w}_{t,0}))] \\
    &\textstyle
    \leq E D G Q \sum_t \mathcal{J}_t \eta_{t-\mathcal{J}_t}^2 \sum_{k=1}^N \pi_k q_k := \frac{C_3}{\ln(1/\lambda(\bm{P}))} < +\infty.
\end{align*}
\end{lemma}
For the proof of Lemma~\ref{lemma:sun}, we introduce the following results:
\begin{flalign}
    &\textstyle
    \label{eq:prop1-3}
    \abs{F_k(\bm{v}) - F_k(\bm{w})} \leq D \cdot \norm{\bm{v}-\bm{w}}, ~\forall \bm{v},\bm{w} \in W, \\
    &\textstyle
    \label{eq:prop1-4}
    \E_{\ar{\mathcal{B}_{t, 0}^{k}}, \dots, \ar{\mathcal{B}^{k}_{t, E{-}1}}} \norm{\bm{w}_{t+1,0} - \bm{w}_{t,0}} \leq \eta_t G E (\sum_{k \in A_t} q_k).
\end{flalign} 
Equation~\eqref{eq:prop1-3} is due to convexity of $\{F_k\}_{k \in \mathcal{K}}$, which gives:
\begin{align*}
    \textstyle
    \scalar{\nabla F_k(\bm{v})}{\bm{v} - \bm{w}} 
    \leq 
    \norm{F_k(\bm{v}) - F_k(\bm{w})}
    \leq 
    \scalar{\nabla F_k(\bm{w})}{\bm{v} - \bm{w}};
\end{align*}
the Cauchy–Schwarz inequality concludes:
\begin{align*}
    \textstyle
    \abs{F_k(\bm{v}) - F_k(\bm{w})}
    &\leq 
    \max \{ \norm{\nabla F_k(\bm{v})}, \norm{\nabla F_k(\bm{w})} \} \norm{\bm{v} - \bm{w}} \\
    &\leq 
    D \cdot \norm{\bm{v} - \bm{w}}.
\end{align*}
Equation~\eqref{eq:prop1-4} follows combining equations~\eqref{eq:G} and~\eqref{eq:update_rule_appendix}:
\begin{align*}
    \textstyle
    \E_{\mathcal{B} | A_t} 
    &\norm{\bm{w}_{t+1,0} - \bm{w}_{t,0}} \leq \\
    &\textstyle
    \leq 
    \eta_t \E_{\mathcal{B} | A_t} \norm{\sum_{k \in A_t} q_k \sum_{j=0}^{E-1} \nabla F_k (\bm{w}_{t,j}^k, \xi_{t,j}^k)} \\
    &\textstyle
    \leq
    \eta_t \sum_{k \in A_t} q_k \sum_{j=0}^{E-1} \E_{\mathcal{B} | A_t} \norm{\nabla F_k (\bm{w}_{t,j}^k, \xi_{t,j}^k)} \\
    &\textstyle
    \leq
    \eta_t G E (\sum_{k \in A_t} q_k).
\end{align*}
\begin{proof}[Proof of Lemma \ref{lemma:sun}]
The evolution of the local objectives after $\mathcal{J}_t$ communication rounds is bounded:
\begin{align*}
    \textstyle
    \sum_t &
    \textstyle
    \eta_t \E [\sum_{k \in A_t} q_k (F_{k}(\bm{w}_{t-\mathcal{J}_t,0}) - F_{k}(\bm{w}_{t,0}))] \\
    &\textstyle
    \stackrel{\text{\tiny (a)}}{\leq} D \sum_t \eta_t \E [\sum_{k \in A_t} q_k \E_{B} \norm{\bm{w}_{t-\mathcal{J}_t,0} - \bm{w}_{t,0}}] \\
    &\textstyle
    \stackrel{\text{\tiny (b)}}{\leq} D \sum_t \eta_t \sum_{d=t-\mathcal{J}_t}^{t-1} \E[\sum_{k \in A_t} q_k \E_{B} \norm{\bm{w}_{d,0} - \bm{w}_{d+1,0}}] \\
    &\textstyle
    \stackrel{\text{\tiny (c)}}{\leq} E D G \sum_t \sum_{d=t-\mathcal{J}_t}^{t-1} \eta_t \eta_d \E[\sum_{k \in A_t} q_k \sum_{k' \in A_d} q_{k'}] \\
    &\textstyle
    \stackrel{\text{\tiny (d)}}{\leq} \frac{E D G}{2} \sum_t \sum_{d=t-\mathcal{J}_t}^{t-1} (\eta_t^2 + \eta_d^2) \E[\sum_{k \in A_t} q_k \sum_{k' \in A_d} q_{k'}] \\
    &\textstyle
    \stackrel{\text{\tiny (e)}}{\leq} E D G Q \sum_t \mathcal{J}_t \eta_{t-\mathcal{J}_t}^2 \sum_{k=1}^N \pi_k q_k := \frac{C_3}{\ln(1/\lambda(\bm{P}))},
\end{align*}
where $(a)$ follows from~\eqref{eq:prop1-3}; $(b)$ applies the triangle inequality; $(c)$ uses~\eqref{eq:prop1-4}; $(d)$ applies the Cauchy–Schwarz inequality; $(e)$ uses $\eta_t < \eta_d \leq \eta_{t-\mathcal{J}_t}$ and $\sum_{k=1}^N q_k = Q$.
\end{proof}

\subsubsection{{Core of the proof}}

The proof consists in two main steps:

$1.\sum_t \eta_t \sum_{k=1}^N \pi_k q_k \E[F_B(\bm{w}_{t - \mathcal{J}_t, 0}) - F_B^*)] {\leq} C_2 {+} \frac{C_3}{\ln(1/\lambda(\bm{P}))}$;

$2.\sum_t \eta_t \sum_{k=1}^N \pi_k q_k \E[F_B(\bm{w}_{t, 0}) {-} F_B(\bm{w}_{t - \mathcal{J}_t, 0})] {\leq} \frac{C_3}{\ln(1/\lambda(\bm{P}))}$.

\emph{Step 1.} Combining Lemma \ref{lemma:li} and \ref{lemma:sun}, we get:
\begin{align*}
    \textstyle
    \sum_t \eta_t \E[\sum\limits_{\scriptscriptstyle k \in A_t} q_k (F_{k}(\bm{w}_{t - \mathcal{J}_t, 0}) - F_{k}(\bm{w}_B^*))]
    \leq C_1 + \frac{C_3}{\ln(1/\lambda(\bm{P}))}.
\end{align*}
The constant $\mathcal{J}_t$, introduced in~\cite{sun2018markov}, is an important parameter for the analysis and frequently used. Combining its definition in Theorem~\ref{theorem:bias} and equation~\eqref{eq:mc_convergence}, it follows:
\begin{align}
    \textstyle
    \abs*{[P^{\mathcal{J}_t}]_{i,j} - \pi_j} \leq C_P \lambda(\bm{P})^{\mathcal{J}_t} \leq \frac{1}{2Ht}, 
    \quad \forall i,j \in [M].
    \label{eq:6.28}
\end{align}
Assume $t \geq T_P$. We derive an important lower bound:
\begin{align}
    &\textstyle
    \E_{A_t \mid A_{t-\mathcal{J}_t}}[\sum_{k \in A_t} q_k (F_k (\bm{w}_{t-\mathcal{J}_t,0}) - F_k(\bm{w}_B^*))] \notag \\
    &\textstyle
    \stackrel{\text{\tiny (a)}}{=} \sum_{\mathcal{I}=1}^M \mathbb{P} (A_t {=} \mathcal{I} {\mid} A_{t-\mathcal{J}_t}) \sum_{k \in \mathcal{I}} q_k (F_k (\bm{w}_{t-\mathcal{J}_t,0}) {-} F_k(\bm{w}_B^*)) \notag \\
    &\textstyle
    \stackrel{\text{\tiny (b)}}{=} \sum_{\mathcal{I}=1}^M ~[P^{\mathcal{J}_t}]_{A_{t-\mathcal{J}_t}, \mathcal{I}} ~\sum_{k \in \mathcal{I}} q_k \left(F_k (\bm{w}_{t-\mathcal{J}_t,0}) - F_k(\bm{w}_B^*)\right) \notag \\
    &\textstyle
    \stackrel{\text{\tiny (c)}}{\geq} \sum_{\mathcal{I}=1}^M ~\left(\pi(\mathcal{I}) - \frac{1}{2Ht}\right) \sum_{k \in \mathcal{I}} q_k (F_k (\bm{w}_{t-\mathcal{J}_t,0}) - F_k(\bm{w}_B^*)) \notag \\
    &\textstyle
    \stackrel{\text{\tiny (d)}}{\geq} (\sum_{k=1}^N \pi_k q_k) \cdot (F_B(\bm{w}_{t-\mathcal{J}_t,0}) - F_B^* ) - \frac{1}{2t} M Q,
\end{align}
where $(a)$ is the definition of the conditional expectation, $(b)$ uses the Markov property, $(c)$ follows from~\eqref{eq:6.28}, and $(d)$ is due to~\eqref{eq:H}. Taking total expectations:
\begin{align}
    \textstyle
    (
    &\textstyle
    \sum_{k=1}^N \pi_k q_k) \sum_t \eta_t \E[F_B(\bm{w}_{t-\mathcal{J}_t,0}) - F_B^*] \notag \\
    &\textstyle
    \leq \sum_t \eta_t \E[\sum_{k \in A_t} q_k (F_k(\bm{w}_{t - \mathcal{J}_t, 0}) - F_k(\bm{w}_B^*))] \notag \\
    &\textstyle
    \quad + \frac{1}{4} M Q \sum_t (\eta_t^2 + \frac{1}{t^2})
    = C_2 + \frac{C_3}{\ln(1/\lambda(\bm{P}))},
\end{align}
where $C_2 = C_1 + \frac{1}{4} M Q \sum_t (\eta_t^2 + \frac{1}{t^2})$.

\emph{Step 2.}
By direct calculation (similar to Lemma \ref{lemma:sun}):
\begin{align*}
    \textstyle
    (\sum_{k=1}^N \pi_k q_k) \sum_t \eta_t \E[F_B(\bm{w}_{t,0}) - F_B(\bm{w}_{t-\mathcal{J}_t,0})] 
    {\leq} \frac{C_3}{\ln(1/\lambda(\bm{P}))}.
\end{align*}

Summing Step 1 and 2, and applying Jensen's inequality:
\begin{align*}
    &\textstyle
    (\sum_{t=1}^T \eta_t) (\sum_{k=1}^N \pi_k q_k) \E[F_B(\bar{\bm{w}}_{T,0}) - F_B^*] \leq \\
    &\textstyle
    (\sum_{k=1}^N \pi_k q_k) \sum_{t=1}^T \eta_t \E[F_B(\bm{w}_{t,0}) - F_B^*]
    \leq C_2 + \frac{2C_3}{\ln(1/\lambda(\bm{P}))},
\end{align*}
where $\bar{\bm{w}}_{T,0} := \frac{\sum_{t=1}^T \eta_t \bm{w}_{t,0}}{\sum_{t=1}^T \eta_t}$, and the constants are in \eqref{eq:convergence_bound}. \qed

\subsection{Proof of Theorem \ref{theorem:total_variation}}
\label{appendix:total_variation}

It follows the same lines of Theorem~\ref{theorem:error}, developing~\eqref{eq:replace_th3} as:
\begin{align*}
    \textstyle
    \norm{\nabla F(\bm{w}_B^*)}
    &\textstyle
    \leq L\sqrt{\frac{2}{\mu}} \sum_{k=1}^N \abs{\alpha_k - p_k} \sqrt{(F_k(\bm{w}_B^*) - F_k^*)} \\
    &\textstyle
    \leq 2 L\sqrt{\frac{2}{\mu}} d_{TV}(\bm{\alpha},\bm{p}) \sqrt{{\Gamma}'},
\end{align*}
where $d_{TV}(\bm{\alpha},\bm{p}) \coloneqq \frac{1}{2} \sum_{k=1}^N \abs{\alpha_k - p_k}$ is the total variation distance between the probability measures $\bm{\alpha}$ and $\bm{p}$.
\qed

\subsection{Minimizing \texorpdfstring{$\epsilon_\text{opt}$}{}}
\label{appendix:opt}

Equation~\ref{eq:convergence_bound} defines the following optimization problem:
\begin{mini*}
{\scriptstyle \bm{q}}{\textstyle f(\bm{q})={\frac{\frac{1}{2} \bm{q}^\intercal \bm{A} \bm{q} + B} {\boldsymbol{\pi}^\intercal \bm{q}} + \textstyle C};}
{}{}
\addConstraint{~\bm{q} \ge 0}
\addConstraint{~\boldsymbol{\pi}^\intercal \bm{q} > 0}
\addConstraint{\textstyle \norm{\bm{q}}_{1} = {Q}.}
\end{mini*}
Let us rewrite the problem by adding a variable $s:=1/\bm{\pi}^\intercal \bm{q}$ and then replacing $\bm{y}:=s\bm{q}$. Note that the objective function is the perspective of a convex function, and is therefore convex:
\begin{mini!}|s|
    {\scriptstyle \bm{y},s}
    {\textstyle f(\bm{y},s) = \frac{1}{2s} \bm{y}^\intercal \bm{A} \bm{y} + Bs + C}{}{}
    \addConstraint{\bm{y} \geq 0, ~s > 0, ~\bm{\pi}^\intercal \bm{y} = 1, ~\norm{\bm{y}}_1=Qs}.
    \label{opt:constraints}
\end{mini!}
The Lagrangian function $\mathcal{L}$ is as follows:
\begin{align}
    \mathcal{L}(\bm{y}, s, \lambda, \theta, \bm{\mu}) = 
    \textstyle
    \frac{1}{2s} \bm{y}^\intercal \bm{A} \bm{y} + Bs + C
    + \notag \\ + \lambda (1 - \bm{\pi}^\intercal \bm{y}) + \theta ( \textstyle{\norm{\bm{y}}_1 - Qs)} - \bm{\mu}^\intercal \bm{y}.
\label{eq:lagrangian}
\end{align}
Since the constraint $s>0$ defines an open set, the set defined by the constraints in~\eqref{opt:constraints} is not closed. However, the solution is never on the boundary $s=0$ because $\mathcal{L}^* \rightarrow +\infty$ as $s \rightarrow 0^+$, and we can consider $s \geq 0$. The KKT conditions for $y_k^*$ read:
\begin{align}
    \textstyle
    \text{if $y_{k}^* > 0$:} ~y_{k}^* = \frac{s^*}{A[{kk}]} (\lambda^* \pi_k - \theta^*); \hfill ~\text{$y_k^* = 0$ otherwise.}
  \label{kkt:y}
\end{align}
Since $\lambda^* \geq 0$, the clients with smaller $\pi_k$ may have $q_k^*=0$.
\vspace{0.2cm}

\subsection{Convexity of \texorpdfstring{$\epsilon_{\text{opt}} + \epsilon_{\text{bias}}$}{}}
\label{appendix:convexity}

In Appendix~\ref{appendix:opt}, we proved that $\epsilon_{\text{opt}}(\bm{q})$ is convex. To prove that $\epsilon_{\text{bias}}(\bm{q})$ is also convex, we need to study the convexity of $\chi^2_{\bm{\alpha} \parallel \bm{p}} = \sum_{\scriptscriptstyle k=1}^{\scriptscriptstyle N} (f_k \circ g_k)(\bm{q})$, where $f_k(p_k) = (p_k - \alpha_k)^2/p_k$, and $g_k(\bm{q}) = (\pi_k q_k)/\sum_{\scriptscriptstyle h=1}^{\scriptscriptstyle N} \pi_h q_h$. We observe that $f_k(p_k)$ is convex, and $g_k(\bm{q})$ is a particular case of linear-fractional function~\cite{boyd2004convex}. By direct inspection, it can be proved that $(f_k \circ g_k)(\bm{q})$ is convex in $\text{dom}(f_k \circ g_k) = \{ \bm{q}: \norm{\bm{q}}_1 = Q>0\}$.
\vspace{0.2cm}

\subsection{Synthetic dataset}
\label{appendix:synthetic}
Our synthetic datasets has been generated as follows:
\begin{enumerate}
    \item For client $k\in\mathcal{K}$, sample group identity~$i_{k}$ from a Bernoulli distribution of parameter $1/2$;
    \item Sample model parameters $\bm{w}^{*} \sim \mathcal{N}(0, I_{d})$ from the $d$-dimensional normal distribution;
    \item For client $k\in\mathcal{K}$ and sample index $j\in\{1,\dots, 150\}$, sample clients input data $\bm{x}_{k}^{(j)} \sim \mathcal{N}(0, I_{d})$ from the $d$-dimensional normal distribution;
    \item For client $k\in\mathcal{K}$ such that $i_{k}=0$ and sample index $j\in\{1,\dots, 150\}$, sample the true labels $y_{k}^{(j)}$ from a Bernoulli distribution with parameter equal to $\text{sigmoid}( \langle \bm{w}^{*}, \bm{x}_{k}^{(j)} \rangle)$;
    \item For client $k\in\mathcal{K}$ such that $i_{k}=1$ and sample index $j\in\{1,\dots, 150\}$, sample the true labels $y_{k}^{(j)}$ from a Bernoulli distribution with parameter equal to $0.8 \cdot \text{sigmoid}(  \langle \bm{w}^{*}, \bm{x}_{k}^{(j)} \rangle)+ 0.2 \cdot (1- \text{sigmoid}(  \langle \bm{w}^{*}, \bm{x}_{k}^{(j)} \rangle ) )$.
\end{enumerate}